\documentclass{article} 
\usepackage{arxiv}

\usepackage{amsmath,amsfonts,bm}

\def\eqref#1{equation~\ref{#1}}

\def\1{\bm{1}}

\DeclareMathAlphabet{\mathsfit}{\encodingdefault}{\sfdefault}{m}{sl}
\SetMathAlphabet{\mathsfit}{bold}{\encodingdefault}{\sfdefault}{bx}{n}

\usepackage{xcolor}
\definecolor{darkblue}{rgb}{0, 0.12, 0.55}
\definecolor{darkgreen}{rgb}{0, 0.55, 0.12}
\definecolor{darkred}{rgb}{0.6,0,0}
\definecolor{darkgreen}{rgb}{0,0.6,0}
\definecolor{Gray}{gray}{0.9}
\usepackage[breaklinks=true,
            colorlinks,
            linkcolor = darkred,
            urlcolor  = darkblue, 
            citecolor = teal,
            bookmarks = false]{hyperref}
\usepackage{url}

\usepackage{graphicx}
\usepackage{wrapfig}
\usepackage{booktabs}
\usepackage{adjustbox}
\usepackage{booktabs}
\usepackage{multirow}
\usepackage{xspace}
\usepackage{array}
\usepackage{bm}
\usepackage{bbm}
\usepackage[ruled,vlined]{algorithm2e}
\usepackage{amsthm,amsmath,amssymb}

\newtheorem{condition}{Condition}
\usepackage{listings}
\usepackage{wrapfig}
\usepackage{subcaption}
\usepackage{url}
\usepackage{colortbl}
\usepackage{adjustbox}

\definecolor{mymauve}{rgb}{0.58,0,0.82}
\newcommand{\resolved}[1]{}

\newcommand{\com}[1]{}

\newcommand{\Votek}{Vote-$k$\xspace}

\newlength{\mysize}

\allowdisplaybreaks

\newtheorem{thm}{Theorem}
\newtheorem{lem}[]{Lemma}
\newtheorem{prop}[thm]{Proposition}

\theoremstyle{definition}

\DeclareMathOperator*{\argmax}{arg\,max}

\title{IDEAL: Influence-Driven Selective Annotations Empower In-Context Learners in Large Language Models}

\author{Shaokun Zhang$^{1}$\thanks{Equal contributions.} \ \ \ \ \ Xiaobo Xia$^{2}$$^*$\thanks{Corresponding authors.} \ \ \ \ \ Zhaoqing Wang$^{2}$ \ \ \ \ \ Ling-Hao Chen$^{3}$ \ \ \ \ \ \\
\textbf{Jiale Liu}$^{4}$ \ \ \ \ \ \textbf{Qingyun Wu}$^{1}$$^{\dag}$ \ \ \ \ \ \textbf{Tongliang Liu}$^{2}$ \\
$^1$Pennsylvania State University \quad\quad
$^2$The University of Sydney \\
$^3$Tsinghua University \quad\quad
$^4$Xidian University \\
\texttt{shaokun.zhang@psu.edu} \ \ \ \ 
\texttt{xiaoboxia.uni@gmail.com} \ \\
\texttt{derrickwang005@gmail.com} \ \ \ \  
\texttt{thu.lhchen@gmail.com} \ \ \ \   
\texttt{leoljl.xdu@gmail.com} \ \\
\texttt{qingyun.wu@psu.edu} \ \ \ \   
\texttt{tongliang.liu@sydney.edu.au} \ \\
}

\begin{document}

\maketitle

\begin{abstract}
In-context learning is a promising paradigm that utilizes in-context examples as prompts for the predictions of large language models. These prompts are crucial for achieving strong performance. However, since the prompts need to be sampled from a large volume of annotated examples, finding the right prompt may result in high annotation costs. To address this challenge, this paper introduces an influence-driven selective annotation method that aims to minimize annotation costs while improving the quality of in-context examples. The essence of our method is
to select a pivotal subset from a large-scale unlabeled data pool to annotate for the subsequent sampling of prompts. Specifically, a directed graph is first constructed to represent unlabeled data. Afterward, the influence of candidate unlabeled subsets is quantified with a diffusion process. A simple yet effective greedy algorithm for unlabeled data selection is lastly introduced. It iteratively selects the data if it provides a maximum marginal gain with respect to quantified influence. Compared with previous efforts on selective annotations, our influence-driven method works in an end-to-end manner, avoids an intractable explicit balance between data diversity and representativeness, and enjoys theoretical support. Experiments confirm the superiority of the proposed method on various benchmarks, achieving better performance under lower time consumption during subset selection. The project page is available at \href{https://skzhang1.github.io/IDEAL/}{https://skzhang1.github.io/IDEAL/}.

\end{abstract}

\section{Introduction}

In-context learning~(ICL) entails presenting a small set of examples with demonstrations as prompts (called in-context examples) to large language models (LLMs), before making predictions on test inputs~\cite{wei2021finetuned,min2022rethinking,akyurek2023learning}. This emerging few-shot learning paradigm is an appealing alternative to supervised fine-tuning as it can avoid heavy parameter updates of language models while improving accuracy~\cite{liu2021makes,yoo2022ground}.

Recent studies indicate that obtaining prompts from a vast collection of annotated examples is crucial to achieving strong performance~\cite{rubin2022learning}. Notably, these studies have illuminated the substantial performance improvements when retrieving analogous examples (under specific embedding criteria) as in-context examples tailored for each individual test input. Since different test scenarios need distinct in-context examples, and each of them is equipped with its pertinent annotations, the necessity of a large volume of annotated examples is emphasized~\cite{su2022selective}. However, obtaining large-scale annotated examples for ICL requires substantial manpower and financial resources. This is because humans not only need to annotate the true label for each example but also need to provide the example demonstration in the annotation process~\cite{wei2022chain}. 

\begin{figure}[h]
    \centering
    \begin{subfigure}[b]{0.48\textwidth}
        \includegraphics[width=\textwidth]{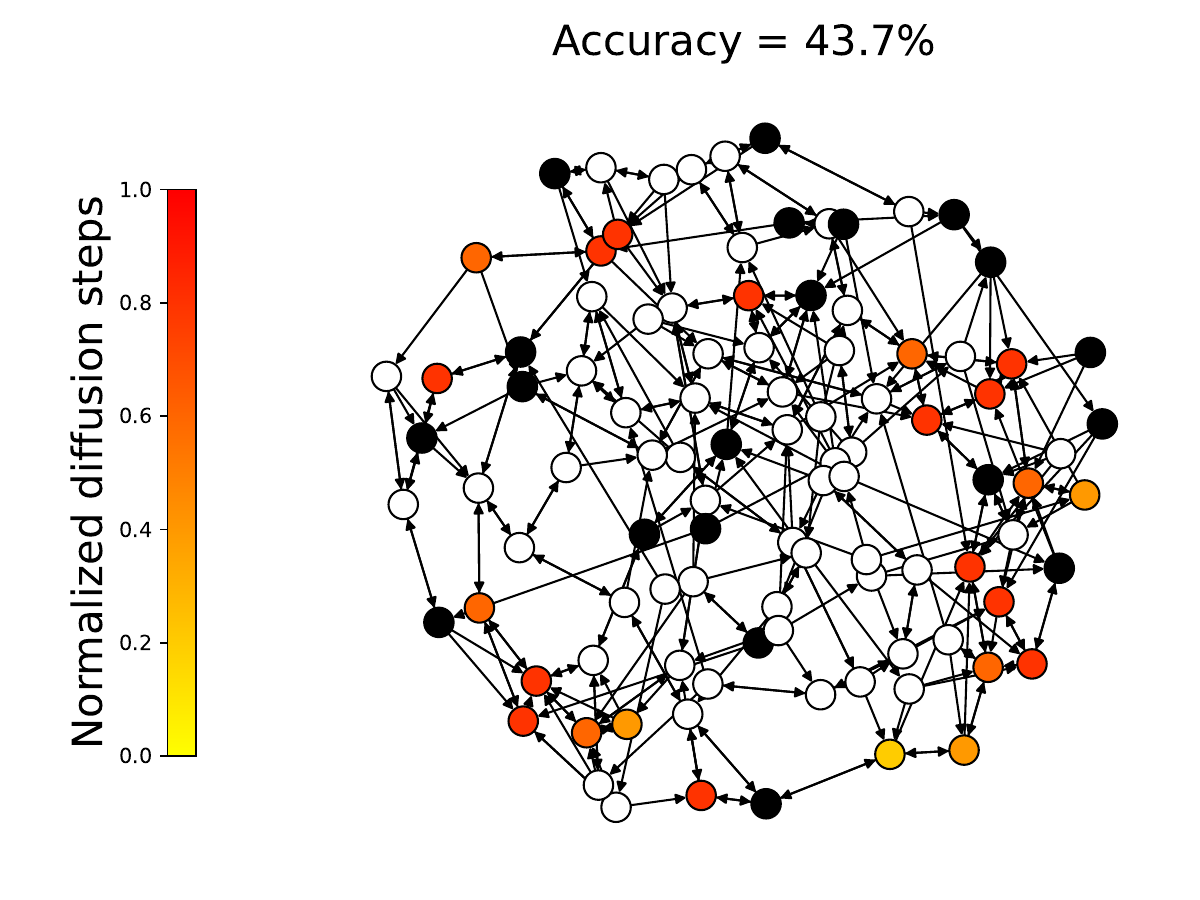}
        \caption{Low-influence subset in unlabeled data.}
    \end{subfigure}
    \hspace{10pt}
    \begin{subfigure}[b]{0.48\textwidth}
    \includegraphics[width=\textwidth]{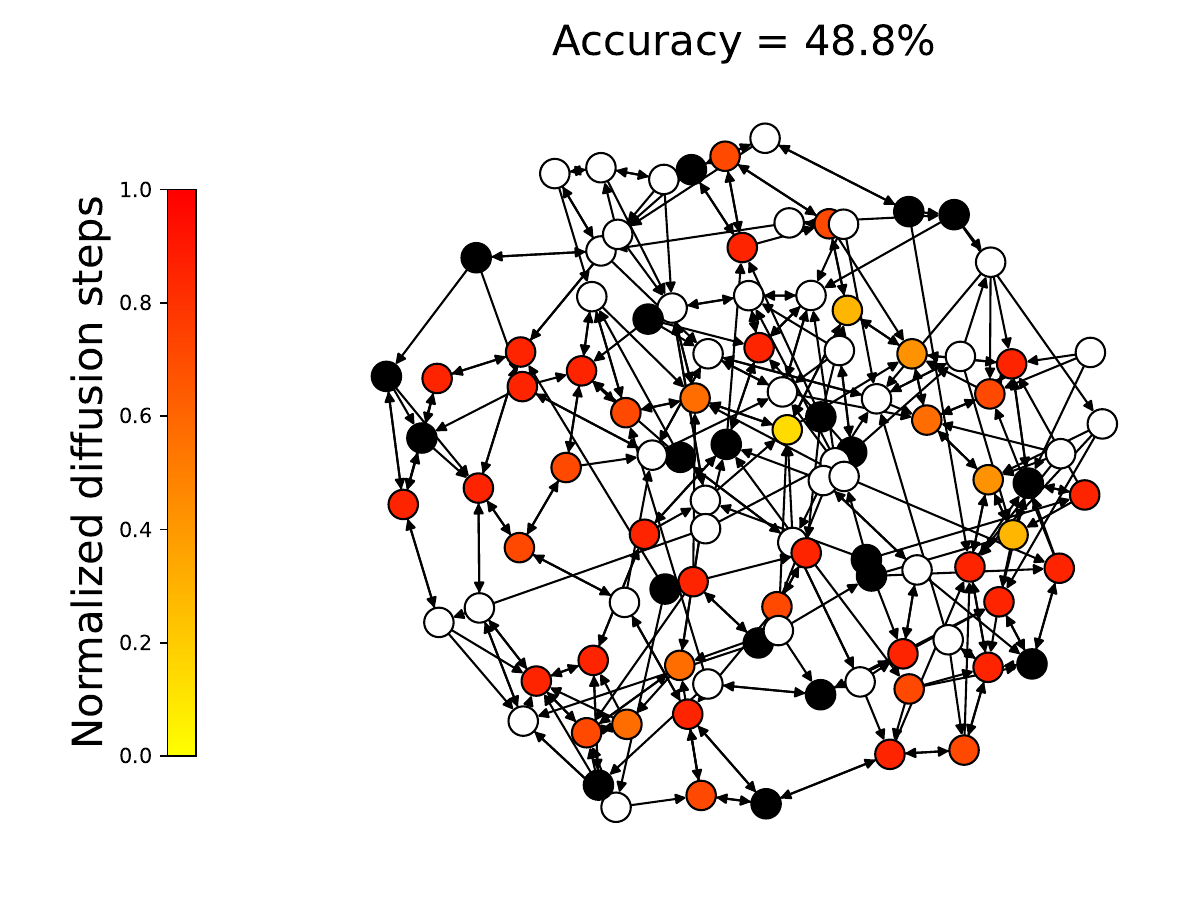}
    \caption{High-influence subset in unlabeled data.}
    \end{subfigure}
    \caption{Visualization of the information diffusion process~\cite{goldenberg2001talk} of two subsets with equal sizes. Experiments are conducted using the SST-5 training set~\cite{socher2013recursive}. To avoid the denseness, we randomly sample 100 examples in total. In this visualization, black nodes present the initial subset without information diffusion. White nodes correspond to the examples that are not influenced by diffusion. For other nodes, darker nodes represent earlier influenced examples. We can observe that the subset with high influence (b) can achieve better performance by influencing a larger group of examples in the unlabeled data pool compared to the subset with low influence (a).}
    \label{fig:demon}
\end{figure}

To reduce the annotation cost, the previous effort Vote-$k$~\cite{su2022selective} made attempts by proposing to select a \textit{diverse} and \textit{representative} subset from a large-scale unlabeled data pool to annotate. Particularly, Vote-$k$ initially selects a small portion of data for diversity and annotates them manually. Then, these annotated data act as prompts for predictions on all other unlabeled data, and choose the remaining ones that need to be annotated, based on diverse confidence scores. However, despite its pretty performance in empirical evaluations, Vote-$k$ is still unsatisfactory in practice. We detail the issues from three aspects. 
(1) The data selection procedure of Vote-$k$ is not end-to-end. This results in inconvenience, increased processing complexity, and added inference costs due to the predictions on unlabeled data. 
(2) Diversity and representativeness need to be balanced carefully~\cite{su2022selective}. Highlighting diversity in data selection is crucial for comprehensive coverage, but may sacrifice representativeness by overlooking exemplary data. Besides, the excessive emphasis on diversity of Vote-$k$ causes the selection of outliers~(see evidence in Appendix~\ref{appendix:umap}). 
(3) Vote-$k$ devoids theoretical guarantees, making it challenging to assess the algorithm's reliability in realistic tasks and constraining its practical utility.

In this paper, to minimize annotation costs for ICL and address the issues of existing work, an innovative data selection method is introduced, where we utilize \underline{i}nfluence-\underline{d}riven s\underline{e}lective \underline{a}nnotations to empower in-context \underline{l}earners~(\textbf{IDEAL}). In essence, IDEAL aims to identify a subset of data that acts as a proxy and closely approximates the vast unlabeled dataset. Once annotated, these selected data can be considered a viable substitute for the large annotated examples in subsequent ICL tasks. In further detail, our method works in an unsupervised and \textit{end-to-end} manner. We first construct a \textit{directed graph}, where its vertices represent unlabeled data and its edges bridge different data based on their similarities. Inspired by influence maximization that aims to select a vertex set at key positions in social graphs~\cite{li2018influence}, we then propose to quantify the influence of each candidate unlabeled subset in our constructed graph, through a classic independent-cascade diffusion model illustrated in Figure~\ref{fig:demon}. To find the subset with high influence, a simple greedy algorithm for unlabeled data selection is introduced. The algorithm does not need a delicate trade-off between diversity and representativeness. Instead, it iteratively selects a vertex if it provides a maximum marginal gain to the influence metric, until the selection is completed based on the annotation budget.

Theoretically, under the influence-driven selective paradigm, we provide the lower bound for the subset influence selected by our method, demonstrating it is at least as large as a certain proportion of the influence of the optimal solution. Empirically, we conduct comprehensive experiments over 9 datasets across diverse tasks (covering classification, commonsense reasoning, dialogue, and text/code generation). Various LLMs and prompt retrieval technologies are included in evaluations. Experimental results demonstrate that our IDEAL can achieve better performance than Vote-$k$, with only 13\% time consumption during subset selection. This creates a strong baseline of selective annotations for follow-up research.

\section{Methodology}

In this section, to reduce the annotation cost of ICL, a framework of influence-driven selective annotations is formulated. 
We discuss how examples should be selected to annotate, leading to better in-context learners for LLMs.

\subsection{Problem setup}
We begin by defining notations and setting up the research problem. Specifically, LLMs perform in-context learning tasks based on a task-specific prompt $\mathbf{Z}=[\mathbf{z}_1,\ldots,\mathbf{z}_c]$, where each $\mathbf{z}_i$ represents one example $(\mathbf{x}_i,y_i)$ consisting of the instance $\mathbf{x}_i$ and label $y_i$, with $c$ examples in total. LLMs generate the prediction for one test input $\mathbf{x}_{\rm{test}}$ conditioned on the prompt $\mathbf{Z}$ followed by $\mathbf{x}_{\rm{test}}$, i.e., $y_{\rm{test}}={\arg\max}_{y\in \mathcal{C}}P(y|\mathbf{Z},\mathbf{x}_{\rm{test}})$, where $\mathcal{C}$ denotes the label space. As each prompt needs distinct annotations, the importance of having a substantial number of annotated examples is stressed, resulting in huge annotation costs. This motivates us to investigate selective annotations. 

Given a pool of unlabeled instances $\mathcal{D}_{\rm{u}}=\{\mathbf{x}_i\}_{i=1}^n$, where $n$ is the number of unlabeled instances, the aim of selective annotations is to select a subset $\mathcal{S}_{\rm{u}}\subset\mathcal{D}_{\rm{u}}$ to make manual annotations, such that performing ICL using prompts retrieved from the selected subset can yield good performance on an unseen test set $\mathcal{D}_{\rm{test}}$. The size of $\mathcal{S}_{\rm{u}}$ is controlled by the annotation budget $m$, i.e., $|\mathcal{S}_{\rm{u}}|=m$.

\subsection{Influence-driven selective annotations}
\label{sec:influence-driven}

\textbf{Overview.} For selective annotations in ICL, we need to identify a subset that approximates vast unlabeled data. Therefore, quantifying the coverage of each candidate subset is critical. To achieve this, we construct a directed graph using the embeddings of unlabeled data and portray their relationships using the edges in the graph. We then quantify the influence of each candidate subset in the constructed graph. An information diffusion model is used for this purpose. After the quantification, we can search the subset with maximum influence, which most closely approximates the unlabeled data. Below we detail the above procedure step by step.

\textbf{Constructing the directed graph.} 
We first compute a vector embedding for each unlabeled instance using Sentence-BERT~\cite{reimers2019sentence}\footnote{\href{https://huggingface.co/sentence-transformers/all-mpnet-base-v2}{https://huggingface.co/sentence-transformers/all-mpnet-base-v2.}}. The obtained embeddings are employed to build a directed graph $\mathcal{G}=(\mathbf{V},\mathbf{E}, \mathbf{P})$, where the vertices $\mathbf{V}=\{\mathbf{v}_i\}_{i=1}^n$ represent the embeddings of the unlabeled instances, $\mathbf{E}$ denotes the set of edges in the graph, and $\mathbf{P}$ denotes the set of weights assigned to edges. In more detail, for each vertex $\mathbf{v}\in\mathbf{V}$, we connect it to its $k$ nearest successors\footnote{In graph theory~\cite{harary2018graph}, a vertex $\mathbf{u}$ is the successor of a vertex $\mathbf{v}$ if it is at the end of an outgoing directed edge ($\mathbf{v}$,$\mathbf{u}$).} in terms of the cosine similarity between the embeddings and then get $\mathbf{E}$. For the edge $(\mathbf{v},\mathbf{u})\in\mathbf{E}$ that connects $\mathbf{v}$ and its successor $\mathbf{u}$, we assign the weight $p(\mathbf{v},\mathbf{u})=\cos(\mathbf{v},\mathbf{u})/{\sum_{\mathbf{z}\in\mathcal{N}(\mathbf{v},k)}\cos(\mathbf{v},\mathbf{z})}$ with $p\in\mathbf{P}$, where $\mathcal{N}(\mathbf{v},k)$ represents the set including $k$ nearest successors of $\mathbf{v}$, and $\cos(\cdot,\cdot)$ is a function that calculates the cosine similarity of two embeddings. The constructed graph depicts the relationships between unlabeled examples in terms of the embedding similarity. 

\textbf{Quantifying subset influence.} 
\begin{figure}[!t]
    \centering    \includegraphics[width = 0.95\hsize]{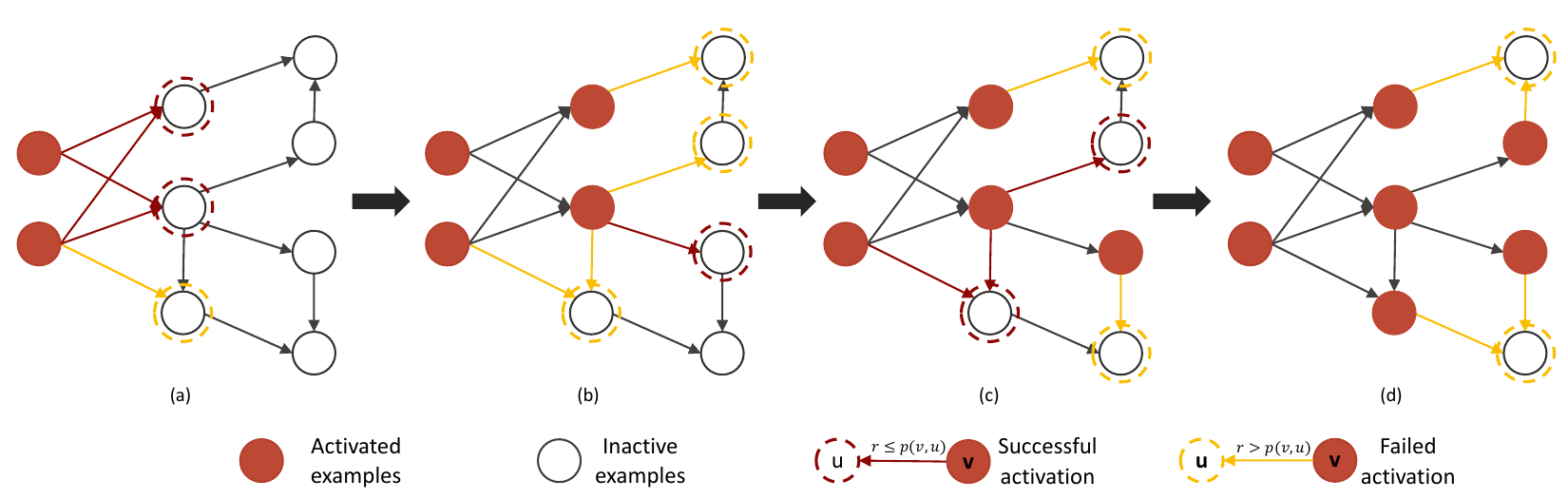}
    \caption{The procedure aims to quantify the influence of each subset of in-context examples. In this procedure, we start with a subset of examples (the red points in (a)). Gradually, the successors of this subset are activated based on the weight $p$ and a random number $r$ sampled from $0$ to $1$. From (a) to (d). The influence of the subset is determined by the number of points that have been activated.}
    \label{fig:system}
\end{figure}
Here we propose to quantify each candidate subset within the constructed graph, which is detailed in Algorithm~\ref{alg:influence_function}. Specifically, given the constructed graph $\mathcal{G}$ and a candidate subset $\mathcal{S}$, the quantification algorithm simulates the \textit{progression of information diffusion} originating from $\mathcal{S}$. The number of influenced vertices can be considered as a measure of the influence of the candidate subset. In other words, the subset that influences more vertices within the 
graph can provide a better approximation of the vast unlabeled data. 
The diffusion process unfolds discretely, progressing through multiple steps. 
\begin{algorithm}[!t]
\caption{Subset influence quantification.}
\label{alg:influence_function}
\SetKwInOut{Input}{Input}
\SetKwInOut{Output}{Output}
\Input{Directed graph $\mathcal{G} = (\mathbf{V}, \mathbf{E}, \mathbf{P})$, subset $\mathcal{S}$.}
\Output{Number of influenced vertices by $\mathcal{S}$ in $\mathcal{G}$.}
$\mathcal{S}_{\rm{active}} \leftarrow \mathcal{S}$, $\mathcal{S}_{\rm{new}}\leftarrow \emptyset$, $L= 0$; \\
\While{$\mathcal{S}_{\rm{active}} \ne \emptyset$}{
    \For{each node $\mathbf{v}$ in $\mathcal{S}_{\rm{active}}$}{
        \For{each successor $\mathbf{u}$ of $\mathbf{v}$ in $\mathcal{G}$}{
            \If{$\mathbf{u}$ not in $\mathcal{S}$}{
                Generate random number $\tau \in [0, 1]$\;
                \If{$\tau \leq p(\mathbf{v},\mathbf{u})$}{
                    $\mathcal{S}\leftarrow \mathcal{S}\cup\mathbf{u}$; $\mathcal{S}_{\rm{new}}\leftarrow \mathcal{S}_{\rm{new}}\cup\mathbf{u}$;
                }
            }
        }
    }
    $\mathcal{S}_{\rm{active}} \leftarrow \mathcal{S}_{\rm{new}}$; $L \leftarrow L + |\mathcal{S}_{\rm{new}}|$; $\mathcal{S}_{\rm{new}} \leftarrow \emptyset$;
}
\Return $L$.
\end{algorithm}
At the beginning, the subset $\mathcal{S}$ is activated. Then at each step, each vertex $\mathbf{v}$ activates its successors that remained inactive in the last step with a probability defined by $p(\mathbf{v},\mathbf{u})$. The activation can be conceptualized as a coin toss where the outcome is determined by the head probability $p(\mathbf{v},\mathbf{u})$. If the result is the head, the vertex $\mathbf{v}$ becomes activated; otherwise, it remains inactive. Starting from $\mathcal{S}$, the diffusion terminates when no further vertex can be activated in the graph. Lastly, we quantify the influence of the set with the number of activated vertices, where a larger number corresponds to greater influence. To help understand the procedure of Algorithm~\ref{alg:influence_function}, we provide an illustration as shown in Figure~\ref{fig:system}. For convenience, we express Algorithm~\ref{alg:influence_function} as an influence function $f_{\mathcal{G}}(\mathcal{S})$ for the graph $\mathcal{G}$ that takes example set $\mathcal{S}$ as inputs, and returns the number of activated vertices $L$.

\begin{algorithm}[!t]
\caption{Searching the subset with maximum influence.}
\label{alg:greedy-influence-maximization}
\SetAlgoNlRelativeSize{-1}
\SetKwInOut{Input}{Input}
\Input{The directed graph $\mathcal{G} = (\mathbf{V}, \mathbf{E},\mathbf{P})$, the annotation budget $m$.}
\KwResult{The set $\mathcal{S}_{\rm{u}}$ that includes $m$ examples to annotate.}
\textbf{Initialize} $\mathcal{S}_{0} \leftarrow \emptyset$, $t=0$; \\
\While{$t < m$}{
    $\mathbf{v}_{t} \leftarrow \argmax_{\mathbf{v} \in \mathbf{V} \setminus \mathcal{S}_{t}} f_{\mathcal{G}}(\mathcal{S}_{t} \cup \{\mathbf{v}\})$; \\
    $\mathcal{S}_{t+1} \leftarrow \mathcal{S}_{t} \cup \mathbf{v}_{t}$;\\
    $t\leftarrow t+1$; 
}
\textbf{Obtain} $\mathcal{S}_{\rm{u}}$ with $\mathcal{S}_m$ using the correspondence between embeddings and instances;  \\
\Return{$\mathcal{S}_{\rm{u}}$.} 
\end{algorithm}
\textbf{Searching the subset with maximum influence.} We exploit a simple yet effective greedy algorithm~\cite{kempe2003maximizing} to search the subset with maximum influence, which is illustrated in Algorithm~\ref{alg:greedy-influence-maximization}. Specifically, the algorithm is initialized with an empty set, and iteratively involves an instance if it can provide the maximum marginal gain to the influence function. The search algorithm terminates when the selected subset meets the annotation budget. Finally, we achieve the set $\mathcal{S}_{\rm{u}}$ that includes $m$ examples to annotate, using the correspondence between embeddings and instances.

\subsection{Prompt retrieval}
After the above influence-driven selective annotations, the subset $\mathcal{S}_{\rm{u}}$ is achieved. By making manual annotations on $\mathcal{S}_{\rm{u}}$, a set of annotated examples is obtained. We can then retrieve examples from the annotated set as in-context examples for each test input. Following previous studies~\cite{liu2021makes,su2022selective}, we will calculate embeddings for all annotated examples using Sentence-BERT~\cite{reimers2019sentence} and identify the most similar instances to each test input based on the cosine similarity. Notice that, the proposed method is agnostic to prompt retrieval methods. As demonstrated in \S\ref{section:retrieve}, our method can be combined with any other prompt retrieval technologies. Better prompt retrieval technologies can further boost final performance.

\section{Theoretical Analysis}
In this section, we perform theoretical analysis on the influence of the subset searched by our algorithm and provide the corresponding lower bound.
For any constructed graph $\mathcal{G}$, we exploit $\psi_{\mathbf{v}}(\mathcal{S})$ to denote the influence improvement of the subset $\mathcal{S}$ after adding $\mathbf{v}$ into $\mathcal{S}$, i.e., $\psi_{\mathbf{v}}(\mathcal{S}) = f_{\mathcal{G}}(\mathcal{S} \cup \mathbf{v}) - f_{\mathcal{G}}(\mathcal{S})$. For convenience, we use $\psi_t=f_{\mathcal{G}}(\mathcal{S}_t)-f_{\mathcal{G}}(\mathcal{S}_{t-1})~(t\geq 1)$ to denote the incremental value of the influence function $f_{\mathcal{G}}$
after adding $\mathbf{v}_t$ into $\mathcal{S}_{t-1}$. Also, we employ $\mathcal{S}_m^*$ to represent the subset with the optimal influence value in the graph $\mathcal{G}$ with annotation budget $m$. Afterward, the optimal solution we expect to search in Algorithm~\ref{alg:greedy-influence-maximization} can be regarded as 
\begin{equation}\label{eq:optimal}
    \mathcal{S}_m^*=\argmax_{\mathcal{S}\subset\mathbf{V}} f_{\mathcal{G}}(\mathcal{S}), \ \ \text{s.t.}\ \ |\mathcal{S}| = m.
\end{equation}
In the following, we present the submodular condition to facilitate theoretical analysis of our method.

\begin{condition}[submodular condition]
\label{condition:submodular}
In the problem of selective annotations, given any graph $\mathcal{G}$ constructed by our procedure, the influence function $f_\mathcal{G}$ is a submodular function which satisfies, for $\forall\mathbf{v}\in\mathbf{V}$, $\forall\mathcal{S}_a\subset\mathcal{S}_b\subset\mathbf{V}$,
\begin{equation}
\label{eq:submodular}
f_{\mathcal{G}}(\mathcal{S}_a \cup \mathbf{v}) - f_{\mathcal{G}}(\mathcal{S}_a)
\geq 
f_{\mathcal{G}}(\mathcal{S}_b \cup \mathbf{v}) - f_{\mathcal{G}}(\mathcal{S}_b).
\end{equation}
\end{condition}
\textbf{Remark~1.}
Intuitively speaking, given any graph $\mathcal{G}$, we say the influence function $f_{\mathcal{G}}$ satisfies the submodular condition if adding one data point to a smaller subset provides more influence than adding the same data point to a larger subset. 
In other words, it reflects the principle of diminishing returns: the marginal gain of including a data point in a set decreases as the size of the set increases. This condition can hold within the influence function~\cite{li2019submodularity}. 
Considering an extreme case, when subset $\mathcal{S} = \mathbf{V}$, the influence improvement of adding any data point to $\mathcal{S}$ will be zero.  
\begin{prop}
\label{prop1}
In Algorithm~\ref{alg:greedy-influence-maximization}, if the influence function $f_{\mathcal{G}}$ satisfies Condition~\ref{condition:submodular}, then for $f_{\mathcal{G}}(S^{*}_{m} )$,  
\begin{equation}
\label{eq:prob1}
\forall t \in [0,m-1), f_{\mathcal{G}}(\mathcal{S}^{*}_{m})   \leq f_{\mathcal{G}}(\mathcal{S}_{t}) + m\psi_{t+1}.
\end{equation}
\end{prop}

\textbf{Remark 2.} Proposition~\ref{prop1} proposes an upper bound for $f_{\mathcal{G}}(\mathcal{S}^{*}_{m})$ in the form of the influence $f_{\mathcal{G}}(\mathcal{S}_{t})$ and its improvement at next step $t+1$, when Algorithm~\ref{alg:greedy-influence-maximization} is applied to selective annotations.

\begin{thm}
\label{them1}
In Algorithm~\ref{alg:greedy-influence-maximization}, if influence function $f_{\mathcal{G}}$ satisfies Condition~\ref{condition:submodular}, when the algorithm terminates at the step $m-1$, $f_\mathcal{G}(\mathcal{S}_{m})$ has a lower bound: 
\begin{equation}
\label{eq:them1}
f_\mathcal{G}(\mathcal{S}_{m}) \geq (1-(1-1/m)^m) f_{\mathcal{G}}(\mathcal{S}^{*}_{m}).
\end{equation}
\end{thm}
\textbf{Remark 3.} 
Theorem~\ref{them1} provides an approximation guarantee for the influence of the selected subset returned by our method. The influence of the selected subset is at least as large as a certain proportion of the influence of the optimal solution, i.e., $1-(1-1/m)^m$. With the annotation budget $m$ increases, this fraction gets closer to $1-1/e$.

For the proofs of Proposition~\ref{prop1} and Theorem~\ref{them1}, readers can refer to Appendix~\ref{app:proofs}.
\section{Experiments}
In this section, we evaluate our method (IDEAL) on multiple datasets that have different categories of tasks. Experimental setups are first introduced~(\S\ref{sec:setup}). We then demonstrate that the proposed method can find a better selective annotation subset in a more efficient way compared with baselines~(\S\ref{sec:main_result}). Moreover, we perform in-depth investigations to provide a better understanding of the superiority of the proposed method~(\S\ref{sec:more_analysis}).
Finally, a case study is also provided to further evaluate the selected subset from our method in an automatic annotation scenario~(\S\ref{sec:automatic_annotation}).
\subsection{Experimental setups}\label{sec:setup}
\textbf{Datasets and tasks.}
Following previous work~\cite{su2022selective}, we employ 9 datasets for the evaluations,
which can be categorized into 4 different tasks, including classification, multi-choice, dialogue, and generation. The details of the datasets are provided in Appendix~\ref{app:datasets}. 
For each dataset, the original ``train/dev/test'' split from the Transformer library~\cite{wolf2019huggingface} is utilized. We use test data for evaluation if they are available publicly (SST-5~\cite{socher2013recursive}, DBpedia~\cite{lehmann2015dbpedia}, MWoZ~\cite{budzianowski2018multiwoz}, and Xsum~\cite{narayan2018don}). Otherwise, we follow the same setting in~\cite{su2022selective} and use the development set. 
We use accuracy as metric for all classifications and multiple choices tasks,  joint accuracy~\cite{budzianowski2018multiwoz} for MWoZ, test suite accuracy~\cite{zhong2020semantic} for GeoQuery~\cite{zelle1996learning}, and ROUGE-L~\cite{lin2004rouge} for Xsum.

\textbf{Models.} 
If not otherwise specified, we run all experiments on the GPT-J 6B model~\cite{gpt-j} except the GeoQuery and MWoZ datasets where we use Text-devinci-002~\cite{chen2021evaluating}. 
We also provide experiments on other models including GPT-Neo 2.7B~\cite{black2021gpt} and more advanced models GPT-3.5-Turbo~\cite{gpt-3.5} in \S\ref{other_evaluation}. Our implementation is detailed in Appendix~\ref{app:imp_details}.

\textbf{Baselines.}
In the main experiments, we perform a comprehensive evaluation of our method that is compared with previous state-of-the-art selective annotation baselines, i.e., Vote-$k$~\cite{su2022selective} and random selection~(abbreviated as ``Random'' below). Note that, in \S\ref{other_method}, we also compare our method with alternative methods that can select a coreset from
large-scale unlabeled data on typical datasets. For the baseline Vote-$k$, we conduct experiments by running its official code\footnote{\href{https://github.com/HKUNLP/icl-selective-annotation}{https://github.com/HKUNLP/icl-selective-annotation.}}.

\subsection{Main results}\label{sec:main_result}
\begin{table}[htb!]
\addtolength{\tabcolsep}{-4.7pt}
\centering
\small
\begin{tabular}{cc   m{0.001em}  ccccc   m{-0.05mm}   c m{-0.05em}  c m{0.001em}   ccc }
\toprule

\multicolumn{2}{c}{\multirow{2}{*}{\textbf{Method}}}
&& \multicolumn{5}{c}{\textbf{Classification}}
&& \multicolumn{1}{c}{\textbf{Multi-Choice}}
&& \multicolumn{1}{c}{\textbf{Dialogue}}
&& \multicolumn{3}{c}{\textbf{Generation}}
\\
\cmidrule(lr){4-8}
\cmidrule(lr){9-10}
\cmidrule(lr){12-13}
\cmidrule(lr){14-16}

&
&& MRPC & SST-5 & MNLI & DBpedia & RTE
&& HellaSwag
&& MWoZ 
&& GeoQ & Xsum
\\

\midrule[.1em]

100
&Random
&& 64.3
&  49.6
& 38.2
&  89.8
&  55.3
&& 66.7
&& 39.9
&& 55.3
& 15.3
\\

100
&Vote-$k$\xspace
&& 64.6
& 46.6
& 38.9
& 89.2
&  57.6
&& 67.9
&& 48.3
&& \textbf{58.8}
& 17.2
\\

\rowcolor{gray!30} 
100
& IDEAL
&& \textbf{66.4}
& \textbf{51.4}
& \textbf{41.0}
&  \textbf{90.6}
&  \textbf{58.9}
&& \textbf{68.6}
&& \textbf{52.2}
&& 58.2
& \textbf{19.9}
\\

\midrule[.05em]

18
&Random
&& 57.4
& 42.9
& 37.8
& 85.2
& 57.9
&& 66.0
&& 37.0
&& 47.5
& 13.6
\\
18
&Vote-$k$\xspace
&& 61.1
& 41.7
& 39.1
& 89.9
& 58.2
&& 66.5
&& 37.7
&& 50.9
& 15.2
\\

\rowcolor{gray!30} 
18
& IDEAL
&& \textbf{63.0}
& \textbf{43.2}
& \textbf{40.0}
&  \textbf{90.1}
&  \textbf{59.4}
&& \textbf{67.1} 
&& \textbf{38.5}
&& \textbf{52.0}
& \textbf{19.6}
\\

\bottomrule
\end{tabular}
\caption{
The performance of our method and baselines on 9 different datasets with an annotation budget of 100 and 18. 
We use similar-based prompt retrieval for all methods and report the average results with 3 different runs for each method. We can observe that our method works better than Random and Vote-$k$ in almost all cases (17/18) under two annotation budgets. The best result in each case is \textbf{bolded}. We also provide the maximum and minimum values of the results in Appendix~\ref{app:max_min}.
}
\label{tab:main_results}
\end{table}
\begin{wrapfigure}{r}{0.60\textwidth}
\centering
\vspace{-12pt}
\includegraphics[width = 0.88\linewidth]{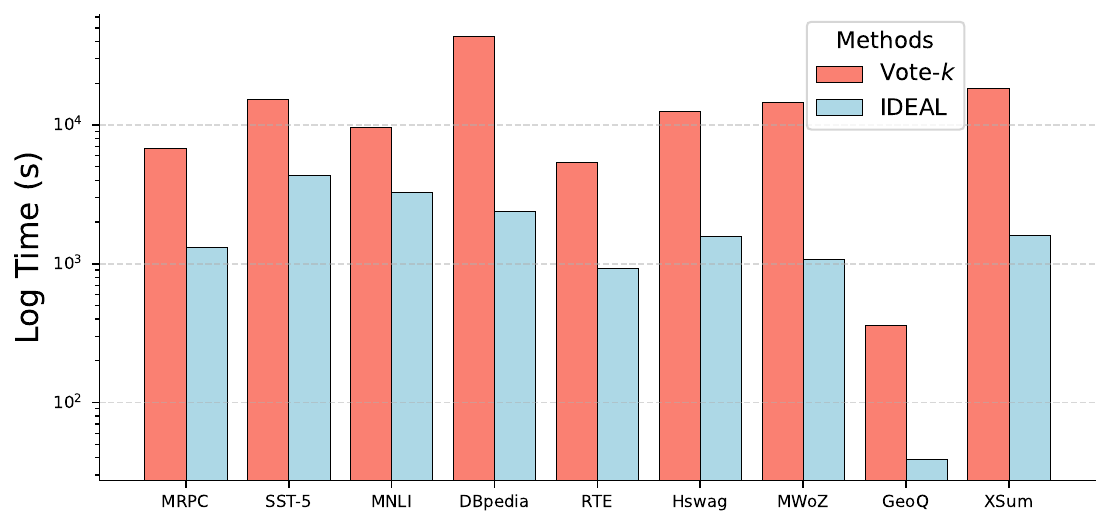}
\caption{Comparison of our method and Vote-$k$ with respect to time consumption during subset selection under the same hardware condition. Here the annotation budget is 18. The y-axis represents the time consumption with a \emph{log scale}. We can observe that our method largely reduces the time cost compared with Vote-$k$. }
\label{fig:cost}
\end{wrapfigure}
\textbf{Measurement on performance.} We first perform the evaluations for Random,  Vote-$k$, and our method. The annotation budget is set to 18 and 100 respectively. 
Note that we include $18$ as the annotation budget considering all annotated examples can be fit to the prompt of the large language models within context limits. Therefore, the prompt retrieve stage can be ignored and the evaluation results can naturally represent the quality of the selected examples. We provide experimental results in Table~\ref{tab:main_results}. As can be seen, our method achieves better performance than baselines in most of the evaluation scenarios (17 out of 18). Interestingly, we find that random selection outperforms Vote-$k$ in 3 out of 18 cases. We conjecture that, under some ideal circumstances, the selected subset by random selection can approximate the distribution of full data. If test data follows the same distribution, good performance can be achieved. Note that we also illustrate selected examples and label distributions in selective annotations in Appendix~\ref{app:select_examples} and Appendix~\ref{appendix:label_statistic} respectively.

\textbf{Measurement on time cost.} 
Previous work Vote-$k$~\cite{su2022selective} encompasses generating prediction for most unlabeled data with a set of selected examples as prompts and performs data selection according to the confidence scores of the prediction. However, this process results in large unnecessary costs at inference time. 
Meanwhile, LLMs are often used as a service and an extra charge will appear with the usage of the token in both the input and output.  In Figure~\ref{fig:cost}, we compare the time cost of subset selection in our method against Vote-$k$ on all tasks with the same hardware. The annotation budget is set to 18. We can observe that our method saves a tremendous amount of cost compared to Vote-$k$. Specifically, under the same hardware conditions, IDEAL achieves a 7.8$\times$ lead on average over Vote-$k$. The speed improvement benefits from the fact that the proposed method does not need to perform example selection by generating predictions on a large number of unlabeled examples and is completely unsupervised.

\subsection{More analysis}\label{sec:more_analysis}

\subsubsection{Larger influence brings better performance}
\begin{wrapfigure}{r}{0.61\textwidth}
    \centering
    \vspace{-15pt}
    \begin{subfigure}[b]{0.29\textwidth}
        \includegraphics[width=\textwidth]{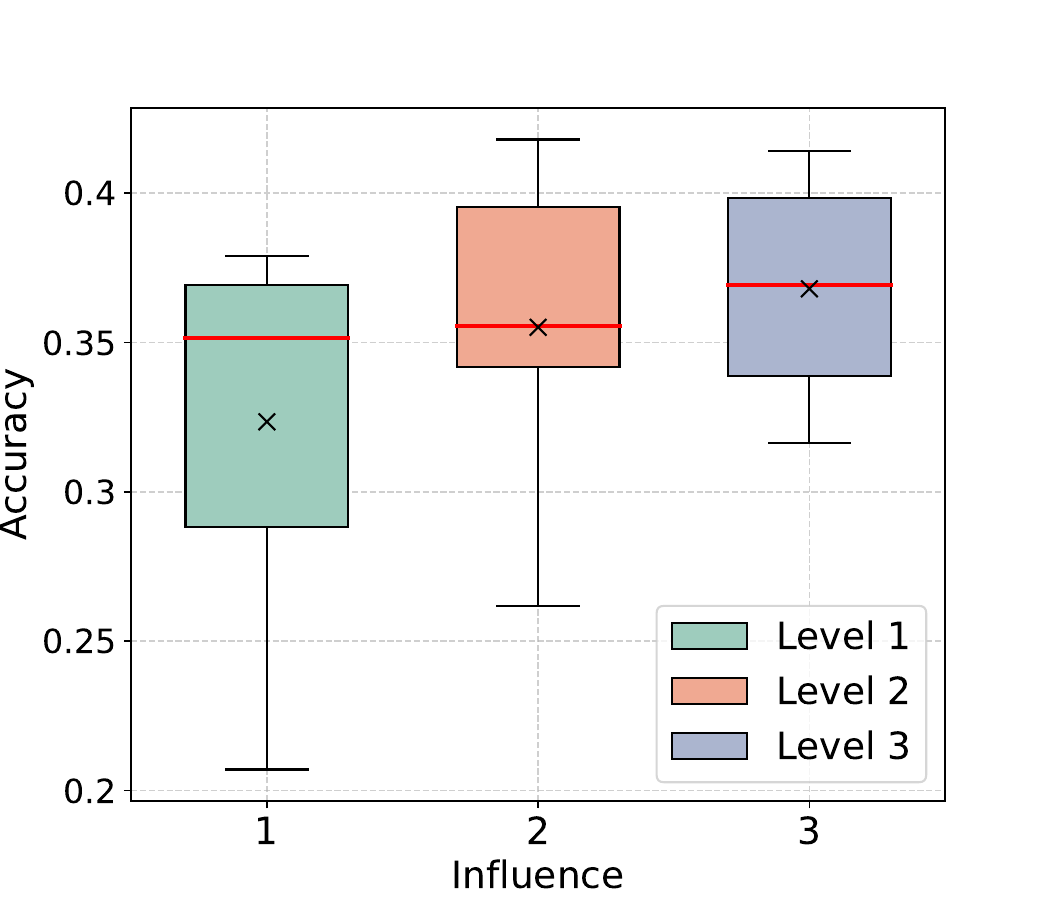}
        \caption{SST-5}
        \label{fig:res:math}
    \end{subfigure}
    \begin{subfigure}[b]{0.29\textwidth}
    \includegraphics[width=\textwidth]{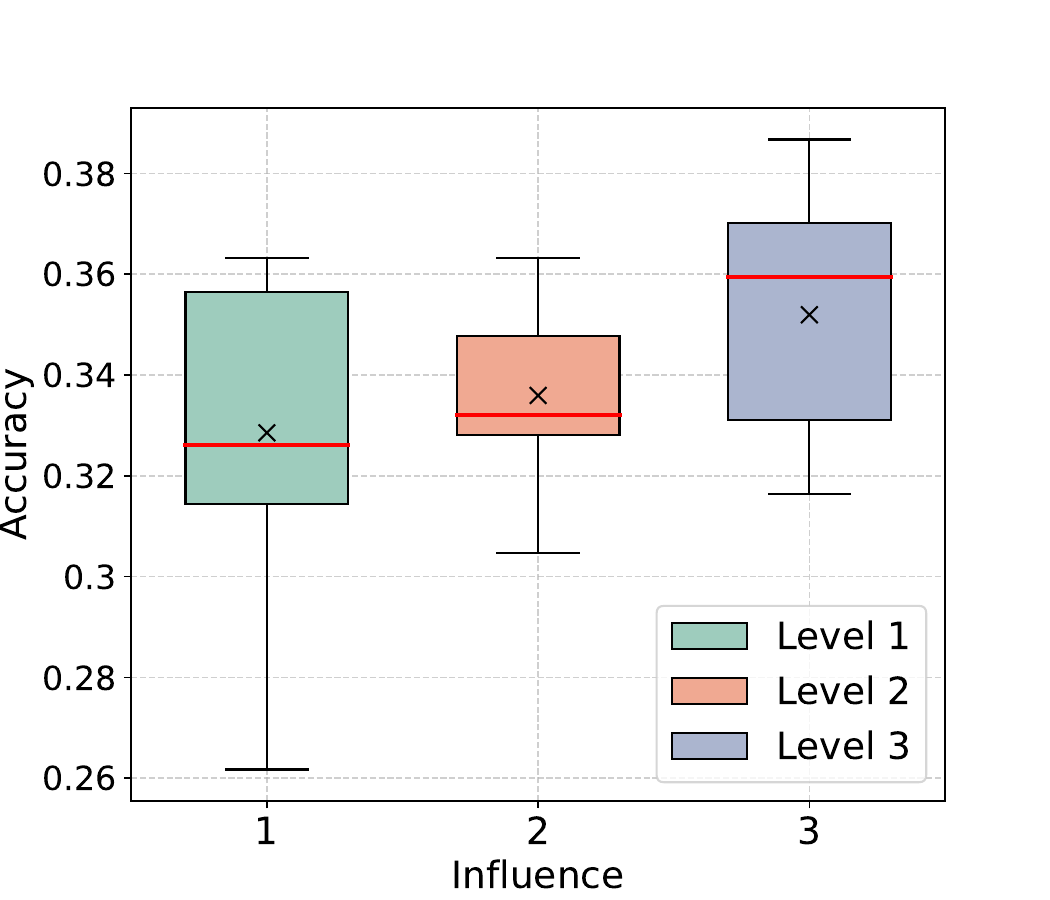}
    \caption{MNLI}
    \label{fig:res:retri}
    \end{subfigure}
    \caption{Influence vs. Performance. The illustration of the positive correlation between the influence achieved by Algorithm~\ref{alg:influence_function} and final performance.}
    \label{fig:results_influence}
\end{wrapfigure}
We conduct experiments to investigate the correlation between subset influence and its corresponding in-context learning performance. Specifically, we randomly select a collection of example subsets from a large unlabeled data pool. We then evaluate each subset as a prompt and record its performance and influence in the constructed graph, resulting in a set of influence-performance pairs. Our goal is to analyze the correlation between these two metrics. To achieve this, we perform experiments on SST-5 and MNLI. We sample 30 subsets and order them according to their influences, where each subset includes 5 examples. We divide this sorted subset sequence equally into three influence levels, with each level containing 10 subsets. We visualize the performance of subsets in each influence level in Figure~\ref{fig:results_influence}.
Our analysis reveals that subsets with larger influence levels achieve better average, median, and worst-case performance. This finding further demonstrates that quantifying the influence of each potential subset is an effective metric in the selective annotation problem.

\subsubsection{Comparisons with alternative methods}
\label{other_method}
\begin{wraptable}{r}{8cm}
\centering
\setlength{\tabcolsep}{1.5pt} 
\vspace{-12pt}
\begin{tabular}{cccccc}
\toprule
Method    & $K$-Means & MFL  & Fast Vote-$k$ & Vote-$k$ & \cellcolor{gray!30} IDEAL         \\ \hline
MRPC      & 57.4    & 58.2 & 59.3          & 61.1     & \cellcolor{gray!30}\textbf{63.0} \\
MNLI      & 35.8    & 38.8 & 39.5          & 39.1     & \cellcolor{gray!30}\textbf{40.0} \\
HellaSwag & 65.4    & 65.2 & 65.9          & 66.5     & \cellcolor{gray!30}\textbf{67.1} \\ \bottomrule
\end{tabular}
\caption{
\label{all_sample_selection_method} 
Comparisons of alternative methods that can select a coreset from large-scale unlabeled data. The annotation budget is 18. Experimental results are reported by averaging over three random trials. The performance of the baseline Vote-$k$ is also included here. The best performance in each case is \textbf{bolded}.
}
\end{wraptable}
We also compare our method with other alternative methods that can select the coreset from large-scale unlabeled data. We perform the evaluations on MRPC, MNLI and HellaSwag. 
We include the following alternative methods
(1) $K$-Means~\cite{lloyd1982least}, which groups all examples into $m$ clusters, and selects the centroid example from each cluster. 
(2) Maximizing facility location (MFL)~\cite{lin2009select}, which aims at optimizing the representativeness of the selected subset. 
(3) Fast Vote-$k$~\cite{su2022selective}, which is an efficient alternative to Vote-$k$ which directly picks $m$ examples with the largest Vote-$k$ scores. 

We show the results in Table~\ref{all_sample_selection_method}. We can observe IDEAL consistently outperforms the baselines in all datasets, demonstrating its superiority. Note that, the graph-based methods (Vote-$k$, Fast Vote-$k$, and our IDEAL) outperform the methods non-graph-based methods~($K$-Means and MFL) in all cases. This phenomenon suggests that graph-based methods are suitable for capturing similarity relationships between examples in the selective annotation problem, which can lead to better results.

\begin{table}[t]
\begin{minipage}{0.49\linewidth}
\centering
\setlength{\tabcolsep}{6pt} 
\begin{tabular}{ccccc}
\toprule
\multicolumn{2}{c}{Method} & \multicolumn{3}{c}{Datasets} \\ \hline
Selection        & Retrieval        & MRPC      & MNLI      & HellaSwag     \\ \hline
Vote-$k$         & Similar          & 64.6      & 38.9      & 67.9          \\
\rowcolor{gray!30} IDEAL            & Similar          & \textbf{66.4}      & \textbf{41.0}      & \textbf{68.6}          \\ \hline
Vote-$k$         & Random           & 60.7      & 37.8      & 64.6          \\
\rowcolor{gray!30} IDEAL            & Random           & \textbf{62.5}      & \textbf{39.0}      & \textbf{66.8}          \\ \bottomrule
\end{tabular}
\caption{Comparison of random and similar prompt retrieval with Vote-$k$ and IDEAL on MRPC, MNLI, and HellaSwag. The subset selection method with a similar prompt retrieve achieves better performance compared with its version with a random prompt retrieve method. The best performance in each case is \textbf{bolded}.}
\label{tab:random_retrieval}
\end{minipage}
\hfill
\begin{minipage}{0.49\linewidth}
\centering
\setlength{\tabcolsep}{12pt} 
\begin{tabular}{cccc}
\toprule
\multirow{2}{*}{Method} & \multirow{2}{*}{Models} & \multicolumn{2}{c}{Test Domain} \\ \cline{3-4} 
                        &                         & IMDb        & BoolQ Cst.        \\ \hline
Vote-$k$                & GPT-Neo                 & 71.1        & 56.4              \\
\rowcolor{gray!30} IDEAL                   & GPT-Neo                 & \textbf{72.2}        & \textbf{58.0}              \\ \hline
Vote-$k$                & GPT-J                   & 76.4        & 56.1              \\
\rowcolor{gray!30} IDEAL                   & GPT-J                   & \textbf{76.8}        & \textbf{56.4}             
\\
\bottomrule
\end{tabular}
\caption{The evaluations on out-of-distribution tasks.  We show the performance of different methods on IMDb and BoolQ Contrast Set (target domains). In the evaluations, the prompts consist of selected SST-2 and BoolQ training examples, respectively (source domains). The best performance in each case is \textbf{bolded}.}
\label{tab:ood}
\end{minipage}
\end{table}
\subsubsection{Evaluation with different retrieval methods}\label{section:retrieve}
In previous experiments, we used a similarity-based prompt retrieval method by default. In this section, we conduct experiments to quantify the effect of different prompt retrieval methods under the annotation 100. We present the results in Table~\ref{tab:random_retrieval}. We observe that both Vote-$k$ and IDEAL suffer from a significant performance drop when the prompt retrieval method is changed from similarity-based to random selection. Notably, IDEAL also achieves better performance than Vote-$k$ when combined with random retrieval in all datasets. It suggests that IDEAL can cultivate a more stable training subset~\cite{chang2023data} for in-context learning tasks. Note that we also show that our IDEAL is more stable and robust against the order of in-context examples in Appendix~\ref{prompts_order}.

\subsubsection{Evaluation on other language models}
\label{other_evaluation}

Here we evaluate IDEAL on other language models, including GPT-Neo 2.7B~\cite{black2021gpt}, and the advanced chat model GPT-3.5-Turbo. While GPT-3.5-Turbo has mainly been optimized for chat, it also performs well on traditional completion tasks~\cite{kheiri2023sentimentgpt}. 
To conduct experiments, we select three classification tasks (MRPC, MNLI, and RTE), considering they are easier for prompting GPT-3.5-Turbo to return responses without pleasantries or explanatory content. 

The evaluation results are presented in Figure~\ref{fig:other_model_all}. Our evaluations reveal that IDEAL consistently outperforms the baselines across all models tested. This demonstrates the versatility of our method in the context of learning tasks using models of varying sizes. 
Notably, we observe that the largest model, i.e., GPT-3.5-Turbo, performs worse than GPT-Neo and GPT-J. This situation arises because GPT-3.5-Turbo is primarily optimized for chat tasks and faces challenges in following human instructions for classification. This scenario also has been identified in \cite{ye2023comprehensive}. 

\subsubsection{Evaluation on out-of-distribution tasks}\label{appendix:umap}
We further evaluate our method on out-of-distribution tasks~\cite{wang2022generalizing, zhou2022domain,zhang2023hypertime,li2022out,huang2023harnessing}, where there is a distribution shift between the selective annotation data and test data. Following~\cite{chang2023data}, we compare IDEAL and Vote-$k$ using SST-2~\cite{socher2013recursive}, BoolQ~\cite{clark2019boolq} as source tasks, and IMDb~\cite{maas2011learning}, BoolQ Contrast Set~\cite{gardner2020evaluating} as target tasks, respectively. In all evaluations, we set the annotation budget as 18 and use the similarity-based retrieve to perform 
the evaluations on the test set in target domains. We use GPT-J 6B and GPT-Neo 2.7B here and show the results in Table~\ref{tab:ood}.
We can observe that IDEAL still outperforms baselines on all datasets with two models, implying that IDEAL could select the subset which could depict the invariant properties of this kind of tasks and generalize to out-of-distribution scenarios.

\begin{figure}[!t]
    \centering
    \begin{subfigure}[b]{0.327\textwidth}
        \includegraphics[width=\textwidth]{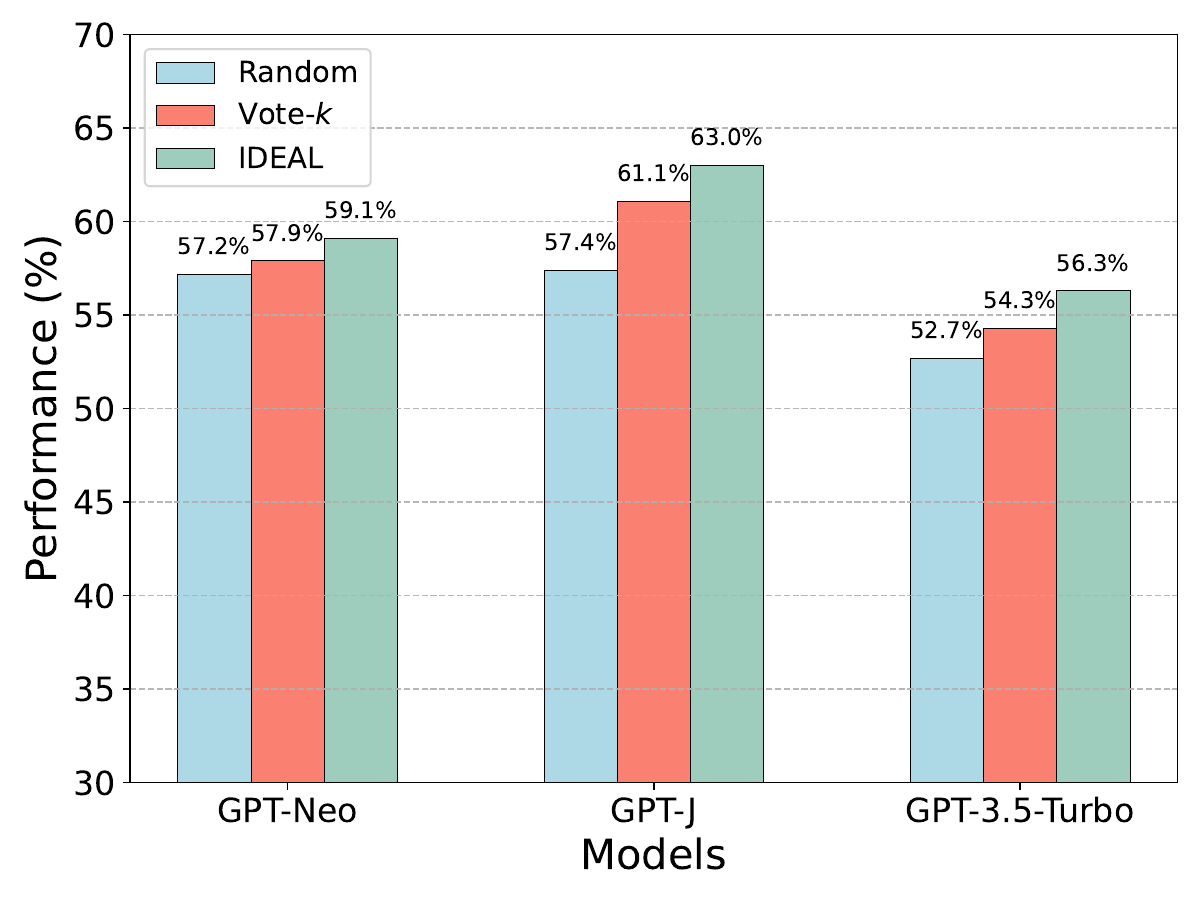}
        \caption{MRPC}
        \label{fig:other_model_mrpc}
    \end{subfigure}
    \begin{subfigure}[b]{0.327\textwidth}
    \includegraphics[width=\textwidth]{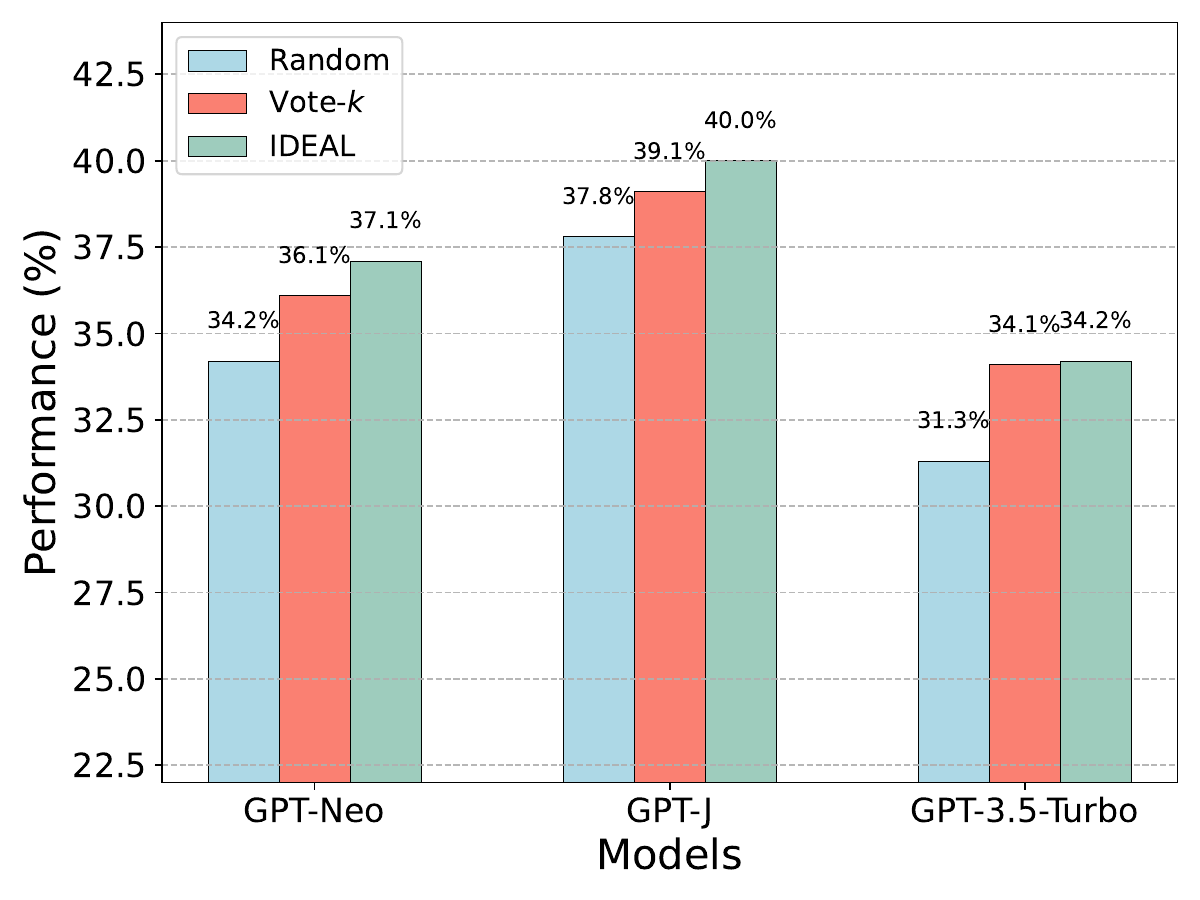}
    \caption{MNLI}
    \label{fig:other_model_mnli}
    \end{subfigure}
    \begin{subfigure}[b]{0.327\textwidth}
    \includegraphics[width=\textwidth]{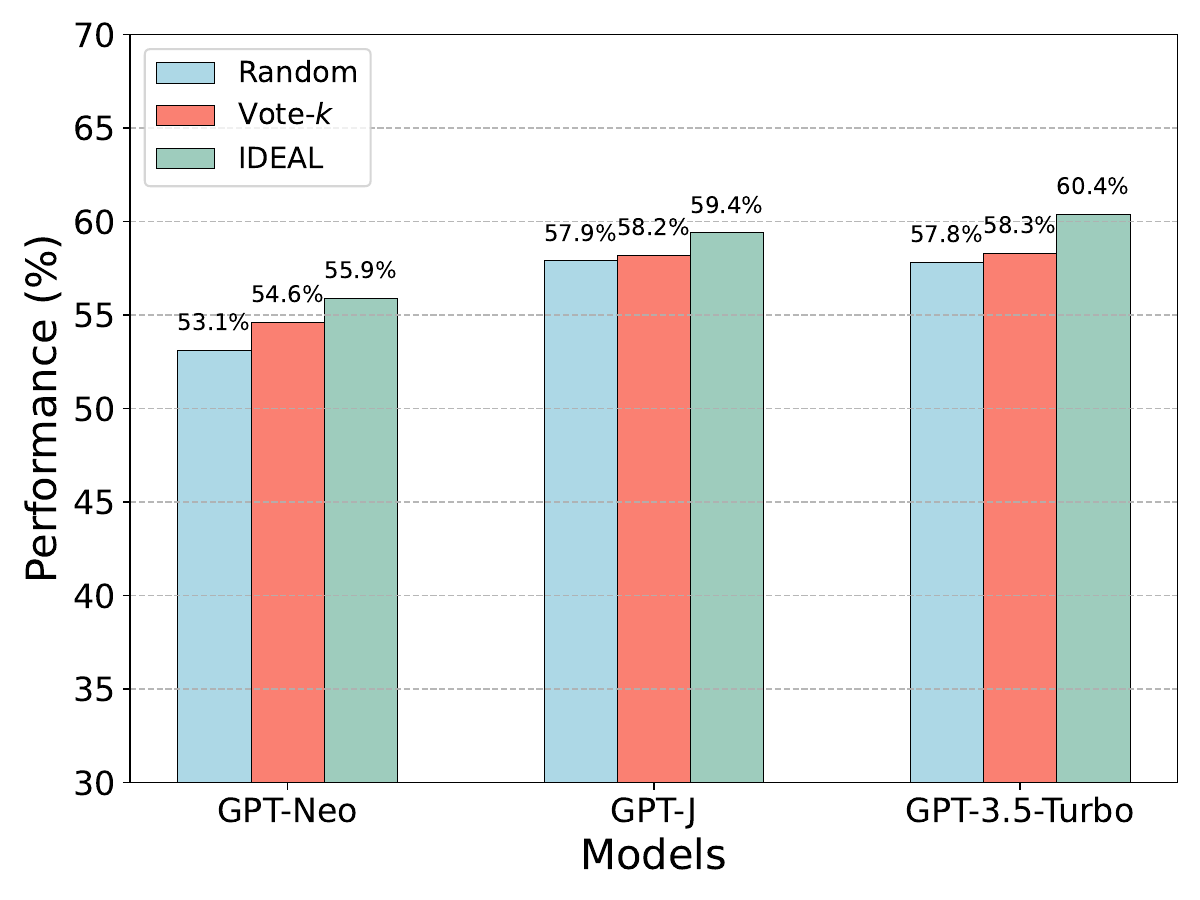}
    \caption{RTE}
    \label{fig:other_model_rte}
    \end{subfigure}
    \caption{Comparisons with various models when the annotation budget is 18. IDEAL consistently achieves the best performance compared with baselines in models with different datasets.}
    \label{fig:other_model_all}
\end{figure}
\subsection{Case study: automatic annotation}
\label{sec:automatic_annotation}
In previous experiments, we used a small set of manually annotated examples as candidate prompts
to make predictions. Instead of this procedure, here we are interested in a case study that utilizes the subset selected by IDEAL to annotate all available unlabeled data automatically, leading to a larger set of candidate prompts. Specifically, we first choose an initial subset from the pool of unlabeled 
data using IDEAL and manually label this selected subset. Afterward, we simulate the information diffusion process from the initial subset to all other data, where we employ those activated data as prompts to predict upcoming activated data at each step and label them accordingly with prediction results. This process ultimately makes a fully labeled training dataset. 
\begin{wraptable}{r}{7.2cm}
\centering
\setlength{\tabcolsep}{1.5pt} 
\begin{tabular}{ccccccc}
\toprule
Method & MRPC  &  SST-5 & MNLI & DBpedia & RTE   \\
\midrule
Vote-$k$   
& 63.8
&  48.6
& 39.5
& 90.2
& 55.7
\\
IDEAL
&  65.2
& 49.4
& \textbf{40.3}
& 90.8
& 57.4
\\
\rowcolor{gray!30} 
Auto-IDEAL 
&  \textbf{65.8}
&  \textbf{50.4}
& 39.8
& \textbf{91.8}
&  \textbf{58.3}
\\
\bottomrule
\end{tabular}
\caption{Comparison between Vote-$k$, IDEAL, and Auto-IDEAL. Auto-IDEAL is an expanded version of IDEAL for automatic annotation. We evaluate these algorithms on all classification tasks and average their performance over three random trials. The best performance in each case is \textbf{bolded}. The results indicate that Auto-IDEAL can enhance the performance of IDEAL and achieve the best performance in 4 out of 5 cases.}
\label{case_study} 
\end{wraptable}
Finally, all examples (including manual labeling and automatic labeling ) are utilized as potential prompts in conjunction with the prompt retrieve technique for final testing.
We name this paradigm as Auto-IDEAL and compare it with Vote-$k$ and origin IDEAL on all classification datasets. We choose 300 training data for each dataset to perform experiments. The manual annotation budget is set to 150, i.e., half of the labels of the candidate prompts in Auto-IDEAL are annotated automatically. 

Experimental results are provided in Table~\ref{case_study}. As can be observed, Auto-IDEAL even achieves better performance than IDEAL in 4 of 5 cases. Notably, although the performance is worse on MNLI, it is still competitive (better than Vote-$k$). It suggests that expanding the candidate prompts through automatic annotation following the diffusion process can further boost the performance of IDEAL. It benefits from the fact that information only diffuses between similar examples. Therefore, unlabeled examples will be automatically annotated using the most similar annotated examples as prompts leading to a promising annotation success rate.

\section{Conclusion}
A series of recent works have confirmed the powerful ability of in-context learning for large language models. We investigate the ability from the perspective of selective annotations and propose an influence-driven method that selects a subset
of data that acts as a proxy and closely approximates full data. Theoretical analysis is provided to establish an upper limit for the global optimal solution, and demonstrate that our greedy search algorithm selects a subset with influence at least as substantial as a specific proportion of the optimal solution's influence. Empirical evaluations illustrate 
the superiority of our method across a range of benchmarks, delivering superior performance while largely reducing the time required for subset selection. We hope this work can help researchers and practitioners understand the promise and potential of selective annotations in in-context learning, and facilitate them in the efficient conceptualization of novel language-based challenges.

\bibliography{iclr2024_conference}

\begin{thebibliography}{10}

\bibitem{wei2021finetuned}
Jason Wei, Maarten Bosma, Vincent~Y Zhao, Kelvin Guu, Adams~Wei Yu, Brian Lester, Nan Du, Andrew~M Dai, and Quoc~V Le.
\newblock Finetuned language models are zero-shot learners.
\newblock In {\em ICLR}, 2022.

\bibitem{min2022rethinking}
Sewon Min, Xinxi Lyu, Ari Holtzman, Mikel Artetxe, Mike Lewis, Hannaneh Hajishirzi, and Luke Zettlemoyer.
\newblock Rethinking the role of demonstrations: What makes in-context learning work?
\newblock In {\em EMNLP}, 2022.

\bibitem{akyurek2023learning}
Ekin Aky{\"u}rek, Dale Schuurmans, Jacob Andreas, Tengyu Ma, and Denny Zhou.
\newblock What learning algorithm is in-context learning? investigations with linear models.
\newblock In {\em ICLR}, 2023.

\bibitem{liu2021makes}
Jiachang Liu, Dinghan Shen, Yizhe Zhang, Bill Dolan, Lawrence Carin, and Weizhu Chen.
\newblock What makes good in-context examples for gpt-3?
\newblock {\em arXiv preprint arXiv:2101.06804}, 2021.

\bibitem{yoo2022ground}
Kang~Min Yoo, Junyeob Kim, Hyuhng~Joon Kim, Hyunsoo Cho, Hwiyeol Jo, Sang-Woo Lee, Sang-goo Lee, and Taeuk Kim.
\newblock Ground-truth labels matter: A deeper look into input-label demonstrations.
\newblock In {\em EMNLP}, 2022.

\bibitem{rubin2022learning}
Ohad Rubin, Jonathan Herzig, and Jonathan Berant.
\newblock Learning to retrieve prompts for in-context learning.
\newblock In {\em NAACL}, 2022.

\bibitem{su2022selective}
Hongjin Su, Jungo Kasai, Chen~Henry Wu, Weijia Shi, Tianlu Wang, Jiayi Xin, Rui Zhang, Mari Ostendorf, Luke Zettlemoyer, Noah~A Smith, et~al.
\newblock Selective annotation makes language models better few-shot learners.
\newblock In {\em ICLR}, 2023.

\bibitem{wei2022chain}
Jason Wei, Xuezhi Wang, Dale Schuurmans, Maarten Bosma, Fei Xia, Ed~Chi, Quoc~V Le, Denny Zhou, et~al.
\newblock Chain-of-thought prompting elicits reasoning in large language models.
\newblock In {\em NeurIPS}, pages 24824--24837, 2022.

\bibitem{goldenberg2001talk}
Jacob Goldenberg, Barak Libai, and Eitan Muller.
\newblock Talk of the network: A complex systems look at the underlying process of word-of-mouth.
\newblock {\em Marketing letters}, 12:211--223, 2001.

\bibitem{socher2013recursive}
Richard Socher, Alex Perelygin, Jean Wu, Jason Chuang, Christopher~D Manning, Andrew~Y Ng, and Christopher Potts.
\newblock Recursive deep models for semantic compositionality over a sentiment treebank.
\newblock In {\em EMNLP}, pages 1631--1642, 2013.

\bibitem{li2018influence}
Yuchen Li, Ju~Fan, Yanhao Wang, and Kian-Lee Tan.
\newblock Influence maximization on social graphs: A survey.
\newblock {\em IEEE Transactions on Knowledge and Data Engineering}, 30(10):1852--1872, 2018.

\bibitem{reimers2019sentence}
Nils Reimers and Iryna Gurevych.
\newblock Sentence-bert: Sentence embeddings using siamese bert-networks.
\newblock {\em arXiv preprint arXiv:1908.10084}, 2019.

\bibitem{harary2018graph}
Frank Harary.
\newblock {\em Graph Theory (on Demand Printing Of 02787)}.
\newblock CRC Press, 2018.

\bibitem{kempe2003maximizing}
David Kempe, Jon Kleinberg, and {\'E}va Tardos.
\newblock Maximizing the spread of influence through a social network.
\newblock In {\em SIGKDD}, pages 137--146, 2003.

\bibitem{li2019submodularity}
Fangqi Li, Chong Di, and Wenwen Xia.
\newblock On the submodularity of diffusion models: Equivalent conditions and applications.
\newblock {\em arXiv preprint arXiv:2002.00845}, 2019.

\bibitem{wolf2019huggingface}
Thomas Wolf, Lysandre Debut, Victor Sanh, Julien Chaumond, Clement Delangue, Anthony Moi, Pierric Cistac, Tim Rault, R{\'e}mi Louf, Morgan Funtowicz, et~al.
\newblock Huggingface's transformers: State-of-the-art natural language processing.
\newblock {\em arXiv preprint arXiv:1910.03771}, 2019.

\bibitem{lehmann2015dbpedia}
Jens Lehmann, Robert Isele, Max Jakob, Anja Jentzsch, Dimitris Kontokostas, Pablo~N Mendes, Sebastian Hellmann, Mohamed Morsey, Patrick Van~Kleef, S{\"o}ren Auer, et~al.
\newblock Dbpedia--a large-scale, multilingual knowledge base extracted from wikipedia.
\newblock {\em Semantic web}, 6(2):167--195, 2015.

\bibitem{budzianowski2018multiwoz}
Pawe{\l} Budzianowski, Tsung-Hsien Wen, Bo-Hsiang Tseng, Inigo Casanueva, Stefan Ultes, Osman Ramadan, and Milica Ga{\v{s}}i{\'c}.
\newblock Multiwoz--a large-scale multi-domain wizard-of-oz dataset for task-oriented dialogue modelling.
\newblock In {\em EMNLP}, 2018.

\bibitem{narayan2018don}
Shashi Narayan, Shay~B Cohen, and Mirella Lapata.
\newblock Don't give me the details, just the summary! topic-aware convolutional neural networks for extreme summarization.
\newblock In {\em EMNLP}, 2018.

\bibitem{zhong2020semantic}
Ruiqi Zhong, Tao Yu, and Dan Klein.
\newblock Semantic evaluation for text-to-sql with distilled test suites.
\newblock In {\em EMNLP}, 2020.

\bibitem{zelle1996learning}
John~M Zelle and Raymond~J Mooney.
\newblock Learning to parse database queries using inductive logic programming.
\newblock In {\em Proceedings of the National Conference on Artificial Intelligence}, pages 1050--1055, 1996.

\bibitem{lin2004rouge}
Chin-Yew Lin.
\newblock Rouge: A package for automatic evaluation of summaries.
\newblock In {\em Text summarization branches out}, pages 74--81, 2004.

\bibitem{gpt-j}
Ben Wang and Aran Komatsuzaki.
\newblock {GPT-J-6B: A 6 Billion Parameter Autoregressive Language Model}.
\newblock \url{https://github.com/kingoflolz/mesh-transformer-jax}, May 2021.

\bibitem{chen2021evaluating}
Mark Chen, Jerry Tworek, Heewoo Jun, Qiming Yuan, Henrique Ponde de~Oliveira Pinto, Jared Kaplan, Harri Edwards, Yuri Burda, Nicholas Joseph, Greg Brockman, et~al.
\newblock Evaluating large language models trained on code.
\newblock {\em arXiv preprint arXiv:2107.03374}, 2021.

\bibitem{black2021gpt}
Sid Black, Leo Gao, Phil Wang, Connor Leahy, and Stella Biderman.
\newblock Gpt-neo: Large scale autoregressive language modeling with mesh-tensorflow.
\newblock 2021.

\bibitem{gpt-3.5}
Openai.
\newblock {GPT-3.5-Turbo}.
\newblock \url{https://platform.openai.com/docs/models/overview}, May 2023.

\bibitem{lloyd1982least}
Stuart Lloyd.
\newblock Least squares quantization in pcm.
\newblock {\em IEEE Transactions on Information Theory}, 28(2):129--137, 1982.

\bibitem{lin2009select}
Hui Lin and Jeff~A Bilmes.
\newblock How to select a good training-data subset for transcription: submodular active selection for sequences.
\newblock In {\em Interspeech}, pages 2859--2862, 2009.

\bibitem{chang2023data}
Ting-Yun Chang and Robin Jia.
\newblock Data curation alone can stabilize in-context learning.
\newblock In {\em ACL}, pages 8123--8144, 2023.

\bibitem{kheiri2023sentimentgpt}
Kiana Kheiri and Hamid Karimi.
\newblock Sentimentgpt: Exploiting gpt for advanced sentiment analysis and its departure from current machine learning.
\newblock {\em arXiv preprint arXiv:2307.10234}, 2023.

\bibitem{ye2023comprehensive}
Junjie Ye, Xuanting Chen, Nuo Xu, Can Zu, Zekai Shao, Shichun Liu, Yuhan Cui, Zeyang Zhou, Chao Gong, Yang Shen, et~al.
\newblock A comprehensive capability analysis of gpt-3 and gpt-3.5 series models.
\newblock {\em arXiv preprint arXiv:2303.10420}, 2023.

\bibitem{wang2022generalizing}
Jindong Wang, Cuiling Lan, Chang Liu, Yidong Ouyang, Tao Qin, Wang Lu, Yiqiang Chen, Wenjun Zeng, and Philip Yu.
\newblock Generalizing to unseen domains: A survey on domain generalization.
\newblock {\em IEEE Transactions on Knowledge and Data Engineering}, 2022.

\bibitem{zhou2022domain}
Kaiyang Zhou, Ziwei Liu, Yu~Qiao, Tao Xiang, and Chen~Change Loy.
\newblock Domain generalization: A survey.
\newblock {\em IEEE Transactions on Pattern Analysis and Machine Intelligence}, 2022.

\bibitem{zhang2023hypertime}
Shaokun Zhang, Yiran Wu, Zhonghua Zheng, Qingyun Wu, and Chi Wang.
\newblock Hypertime: Hyperparameter optimization for combating temporal distribution shifts.
\newblock {\em arXiv preprint arXiv:2305.18421}, 2023.

\bibitem{li2022out}
Yewen Li, Chaojie Wang, Xiaobo Xia, Tongliang Liu, Xin Miao, and Bo~An.
\newblock Out-of-distribution detection with an adaptive likelihood ratio on informative hierarchical vae.
\newblock In {\em NeurIPS}, pages 7383--7396, 2022.

\bibitem{huang2023harnessing}
Zhuo Huang, Xiaobo Xia, Li~Shen, Bo~Han, Mingming Gong, Chen Gong, and Tongliang Liu.
\newblock Harnessing out-of-distribution examples via augmenting content and style.
\newblock In {\em ICLR}, 2023.

\bibitem{clark2019boolq}
Christopher Clark, Kenton Lee, Ming-Wei Chang, Tom Kwiatkowski, Michael Collins, and Kristina Toutanova.
\newblock Boolq: Exploring the surprising difficulty of natural yes/no questions.
\newblock In {\em NAACL-HLT}, pages 2924--2936, 2019.

\bibitem{maas2011learning}
Andrew Maas, Raymond~E Daly, Peter~T Pham, Dan Huang, Andrew~Y Ng, and Christopher Potts.
\newblock Learning word vectors for sentiment analysis.
\newblock In {\em Proceedings of the 49th Annual Meeting of the Association for Computational Linguistics: Human Language Technologies}, pages 142--150, 2011.

\bibitem{gardner2020evaluating}
Matt Gardner, Yoav Artzi, Victoria Basmova, Jonathan Berant, Ben Bogin, Sihao Chen, Pradeep Dasigi, Dheeru Dua, Yanai Elazar, Ananth Gottumukkala, et~al.
\newblock Evaluating models' local decision boundaries via contrast sets.
\newblock In {\em Findings of EMNLP}, 2020.

\bibitem{dong2022survey}
Qingxiu Dong, Lei Li, Damai Dai, Ce~Zheng, Zhiyong Wu, Baobao Chang, Xu~Sun, Jingjing Xu, and Zhifang Sui.
\newblock A survey for in-context learning.
\newblock {\em arXiv preprint arXiv:2301.00234}, 2023.

\bibitem{xie2022explanation}
Sang~Michael Xie, Aditi Raghunathan, Percy Liang, and Tengyu Ma.
\newblock An explanation of in-context learning as implicit bayesian inference.
\newblock In {\em ICLR}, 2022.

\bibitem{shin2022effect}
Seongjin Shin, Sang-Woo Lee, Hwijeen Ahn, Sungdong Kim, HyoungSeok Kim, Boseop Kim, Kyunghyun Cho, Gichang Lee, Woomyoung Park, Jung-Woo Ha, et~al.
\newblock On the effect of pretraining corpora on in-context learning by a large-scale language model.
\newblock In {\em NAACL}, 2022.

\bibitem{zhang2023trained}
Ruiqi Zhang, Spencer Frei, and Peter~L Bartlett.
\newblock Trained transformers learn linear models in-context.
\newblock {\em arXiv preprint arXiv:2306.09927}, 2023.

\bibitem{bai2023transformers}
Yu~Bai, Fan Chen, Huan Wang, Caiming Xiong, and Song Mei.
\newblock Transformers as statisticians: Provable in-context learning with in-context algorithm selection.
\newblock {\em arXiv preprint arXiv:2306.04637}, 2023.

\bibitem{kim2022self}
Hyuhng~Joon Kim, Hyunsoo Cho, Junyeob Kim, Taeuk Kim, Kang~Min Yoo, and Sang-goo Lee.
\newblock Self-generated in-context learning: Leveraging auto-regressive language models as a demonstration generator.
\newblock {\em arXiv preprint arXiv:2206.08082}, 2022.

\bibitem{wang2022iteratively}
Boshi Wang, Xiang Deng, and Huan Sun.
\newblock Iteratively prompt pre-trained language models for chain of thought.
\newblock In {\em EMNLP}, 2022.

\bibitem{mishra2022cross}
Swaroop Mishra, Daniel Khashabi, Chitta Baral, and Hannaneh Hajishirzi.
\newblock Cross-task generalization via natural language crowdsourcing instructions.
\newblock In {\em ACL}, 2022.

\bibitem{bansal2022rethinking}
Hritik Bansal, Karthik Gopalakrishnan, Saket Dingliwal, Sravan Bodapati, Katrin Kirchhoff, and Dan Roth.
\newblock Rethinking the role of scale for in-context learning: An interpretability-based case study at 66 billion scale.
\newblock In {\em ACL}, 2022.

\bibitem{chan2022data}
Stephanie Chan, Adam Santoro, Andrew Lampinen, Jane Wang, Aaditya Singh, Pierre Richemond, James McClelland, and Felix Hill.
\newblock Data distributional properties drive emergent in-context learning in transformers.
\newblock In {\em NeurIPS}, pages 18878--18891, 2022.

\bibitem{li2023transformers}
Yingcong Li, Muhammed~Emrullah Ildiz, Dimitris Papailiopoulos, and Samet Oymak.
\newblock Transformers as algorithms: Generalization and stability in in-context learning.
\newblock In {\em ICML}, 2023.

\bibitem{garg2022can}
Shivam Garg, Dimitris Tsipras, Percy~S Liang, and Gregory Valiant.
\newblock What can transformers learn in-context? a case study of simple function classes.
\newblock In {\em NeurIPS}, pages 30583--30598, 2022.

\bibitem{srivastava2022beyond}
Aarohi Srivastava, Abhinav Rastogi, Abhishek Rao, Abu Awal~Md Shoeb, Abubakar Abid, Adam Fisch, Adam~R Brown, Adam Santoro, Aditya Gupta, Adri{\`a} Garriga-Alonso, et~al.
\newblock Beyond the imitation game: Quantifying and extrapolating the capabilities of language models.
\newblock {\em Transactions on Machine Learning Research}, 2023.

\bibitem{wang2022super}
Yizhong Wang, Swaroop Mishra, Pegah Alipoormolabashi, Yeganeh Kordi, Amirreza Mirzaei, Anjana Arunkumar, Arjun Ashok, Arut~Selvan Dhanasekaran, Atharva Naik, David Stap, et~al.
\newblock Super-naturalinstructions: Generalization via declarative instructions on 1600+ nlp tasks.
\newblock In {\em EMNLP}, 2022.

\bibitem{chen2022meta}
Yanda Chen, Ruiqi Zhong, Sheng Zha, George Karypis, and He~He.
\newblock Meta-learning via language model in-context tuning.
\newblock In {\em ACL}, 2022.

\bibitem{lee2022does}
Young-Jun Lee, Chae-Gyun Lim, and Ho-Jin Choi.
\newblock Does gpt-3 generate empathetic dialogues? a novel in-context example selection method and automatic evaluation metric for empathetic dialogue generation.
\newblock In {\em COLING}, pages 669--683, 2022.

\bibitem{cho2023prompt}
Hyunsoo Cho, Hyuhng~Joon Kim, Junyeob Kim, Sang-Woo Lee, Sang-goo Lee, Kang~Min Yoo, and Taeuk Kim.
\newblock Prompt-augmented linear probing: Scaling beyond the limit of few-shot in-context learners.
\newblock In {\em AAAI}, pages 12709--12718, 2023.

\bibitem{huang2018epsilon}
Lingxiao Huang, Shaofeng H-C Jiang, Jian Li, and Xuan Wu.
\newblock Epsilon-coresets for clustering (with outliers) in doubling metrics.
\newblock In {\em FOCS}, pages 814--825, 2018.

\bibitem{huang2023near}
Lingxiao Huang, Shaofeng H-C Jiang, Jianing Lou, and Xuan Wu.
\newblock Near-optimal coresets for robust clustering.
\newblock In {\em ICLR}, 2023.

\bibitem{feldman2020neural}
Vitaly Feldman and Chiyuan Zhang.
\newblock What neural networks memorize and why: Discovering the long tail via influence estimation.
\newblock In {\em NeurIPS}, pages 2881--2891, 2020.

\bibitem{sorscher2022beyond}
Ben Sorscher, Robert Geirhos, Shashank Shekhar, Surya Ganguli, and Ari Morcos.
\newblock Beyond neural scaling laws: beating power law scaling via data pruning.
\newblock In {\em NeurIPS}, pages 19523--19536, 2022.

\bibitem{sener2018active}
Ozan Sener and Silvio Savarese.
\newblock Active learning for convolutional neural networks: A core-set approach.
\newblock In {\em ICLR}, 2018.

\bibitem{toneva2019empirical}
Mariya Toneva, Alessandro Sordoni, Remi Tachet~des Combes, Adam Trischler, Yoshua Bengio, and Geoffrey~J Gordon.
\newblock An empirical study of example forgetting during deep neural network learning.
\newblock In {\em ICLR}, 2019.

\bibitem{he2023large}
Muyang He, Shuo Yang, Tiejun Huang, and Bo~Zhao.
\newblock Large-scale dataset pruning with dynamic uncertainty.
\newblock {\em arXiv preprint arXiv:2306.05175}, 2023.

\bibitem{wang2018dataset}
Tongzhou Wang, Jun-Yan Zhu, Antonio Torralba, and Alexei~A Efros.
\newblock Dataset distillation.
\newblock {\em arXiv preprint arXiv:1811.10959}, 2018.

\bibitem{zhao2021dataset}
Bo~Zhao, Konda~Reddy Mopuri, and Hakan Bilen.
\newblock Dataset condensation with gradient matching.
\newblock In {\em ICLR}, 2021.

\bibitem{shin2023loss}
Seungjae Shin, Heesun Bae, Donghyeok Shin, Weonyoung Joo, and Il-Chul Moon.
\newblock Loss-curvature matching for dataset selection and condensation.
\newblock In {\em AISTATS}, pages 8606--8628, 2023.

\bibitem{cui2023scaling}
Justin Cui, Ruochen Wang, Si~Si, and Cho-Jui Hsieh.
\newblock Scaling up dataset distillation to imagenet-1k with constant memory.
\newblock In {\em ICML}, pages 6565--6590, 2023.

\bibitem{du2023minimizing}
Jiawei Du, Yidi Jiang, Vincent~YF Tan, Joey~Tianyi Zhou, and Haizhou Li.
\newblock Minimizing the accumulated trajectory error to improve dataset distillation.
\newblock In {\em CVPR}, pages 3749--3758, 2023.

\bibitem{loo2023dataset}
Noel Loo, Ramin Hasani, Mathias Lechner, and Daniela Rus.
\newblock Dataset distillation with convexified implicit gradients.
\newblock In {\em ICML}, 2023.

\bibitem{xia2023moderate}
Xiaobo Xia, Jiale Liu, Jun Yu, Xu~Shen, Bo~Han, and Tongliang Liu.
\newblock Moderate coreset: A universal method of data selection for real-world data-efficient deep learning.
\newblock In {\em ICLR}, 2023.

\bibitem{yang2023dataset}
Shuo Yang, Zeke Xie, Hanyu Peng, Min Xu, Mingming Sun, and Ping Li.
\newblock Dataset pruning: Reducing training data by examining generalization influence.
\newblock In {\em ICLR}, 2023.

\bibitem{rolnick2014greedy}
David Rolnick and Jonathan Weed.
\newblock Greedy maximization of submodular functions.
\newblock 2014.

\bibitem{mcinnes2018umap-software}
Leland McInnes, John Healy, Nathaniel Saul, and Lukas Grossberger.
\newblock Umap: Uniform manifold approximation and projection.
\newblock {\em The Journal of Open Source Software}, 3(29):861, 2018.

\bibitem{lu2021fantastically}
Yao Lu, Max Bartolo, Alastair Moore, Sebastian Riedel, and Pontus Stenetorp.
\newblock Fantastically ordered prompts and where to find them: Overcoming few-shot prompt order sensitivity.
\newblock {\em arXiv preprint arXiv:2104.08786}, 2021.

\bibitem{dolan2004unsupervised}
William Dolan, Chris Quirk, Chris Brockett, and Bill Dolan.
\newblock Unsupervised construction of large paraphrase corpora: Exploiting massively parallel news sources.
\newblock In {\em ACL}, 2004.

\bibitem{williams2017broad}
Adina Williams, Nikita Nangia, and Samuel~R Bowman.
\newblock A broad-coverage challenge corpus for sentence understanding through inference.
\newblock {\em arXiv preprint arXiv:1704.05426}, 2017.

\bibitem{bentivogli2009fifth}
Luisa Bentivogli, Peter Clark, Ido Dagan, and Danilo Giampiccolo.
\newblock The fifth pascal recognizing textual entailment challenge.
\newblock {\em TAC}, 7:8, 2009.

\bibitem{zellers2019hellaswag}
Rowan Zellers, Ari Holtzman, Yonatan Bisk, Ali Farhadi, and Yejin Choi.
\newblock Hellaswag: Can a machine really finish your sentence?
\newblock In {\em ACL}, 2019.

\bibitem{paszke2019pytorch}
Adam Paszke, Sam Gross, Francisco Massa, Adam Lerer, James Bradbury, Gregory Chanan, Trevor Killeen, Zeming Lin, Natalia Gimelshein, Luca Antiga, et~al.
\newblock Pytorch: An imperative style, high-performance deep learning library.
\newblock In {\em NeurIPS}, 2019.

\bibitem{kwiatkowski2019natural}
Tom Kwiatkowski, Jennimaria Palomaki, Olivia Redfield, Michael Collins, Ankur Parikh, Chris Alberti, Danielle Epstein, Illia Polosukhin, Jacob Devlin, Kenton Lee, et~al.
\newblock Natural questions: a benchmark for question answering research.
\newblock {\em Transactions of the Association for Computational Linguistics}, 7:453--466, 2019.

\end{thebibliography}
\bibliographystyle{unsrt}

\appendix

\section{Related Work}
\subsection{In-context learning}
In-context learning~(ICL) has become a new paradigm for natural language processing (NLP), where large language models make predictions only based on contexts augmented with a few examples~\cite{dong2022survey,xie2022explanation,shin2022effect,zhang2023trained,bai2023transformers}. A series of works attempts to revise, enhance, and understand ICL, which include but are not limited to prompt tuning~\cite{kim2022self,wang2022iteratively,mishra2022cross}, analyzing intrinsic mechanism~\cite{bansal2022rethinking,chan2022data,li2023transformers,garg2022can}, evaluations~\cite{srivastava2022beyond,wang2022super}, applications in multiple domains~\cite{chen2022meta,lee2022does,cho2023prompt}, and etc. Different from them, this paper studies selective annotations for ICL, which can effectively reduce the annotation cost in ICL. Furthermore, compared with recent work~\cite{su2022selective}, as discussed in the main paper, this work is superior in many aspects, such as the end-to-end manner, mitigation of the trade-off between diversity and representativeness, theoretical guarantees, and better empirical performance. 

\subsection{Coreset selection}
Coreset selection focuses on selecting a small but highly
informative subset from a large dataset for follow-up tasks, which can significantly reduce the data storage cost and training consumption~\cite{huang2018epsilon,huang2023near,feldman2020neural,sorscher2022beyond}. Most of the works on coreset selection target the scenes of supervised learning and classification~\cite{sener2018active,toneva2019empirical,he2023large}. Only a few works extend coreset selection into unsupervised cases~\cite{sorscher2022beyond,su2022selective}. This paper studies unsupervised data selection for annotations in ICL, which reduces the annotation expenses of prompts and helps large language models become better few-shot learners. Also, it enjoys theoretical support. Therefore, this work is different from previous efforts and contributes to the research community.

\subsection{Data distillation}
Data distillation~\cite{wang2018dataset,zhao2021dataset,shin2023loss,cui2023scaling,du2023minimizing,loo2023dataset} is an alternative approach for dataset compression and curation, which is inspired by knowledge distillation. Different from coreset selection, this series of works target \textit{synthesizing} a small but informative dataset as an alternative to the original dataset. However, data distillation is criticized for only synthesizing a small number of data points due to computational source limitations~\cite{xia2023moderate,yang2023dataset}. The performances of data distillation and data selection are therefore not compared directly. Besides, it is under-explored about how to perform data distillation in an unsupervised manner on natural language processing tasks. Based on this analysis, the data distillation strategy is not involved in empirical evaluations. 

\section{Proofs}\label{app:proofs}
\subsection{Preliminary theoretical results}
We first present some preliminary theoretical results and their corresponding proofs for the subsequential proofs of Proposition~\ref{prop1} and Theorem~\ref{them1}. 
\subsubsection{Lemma~1}
\begin{lem}\label{lem:1}
    Given a graph $\mathcal{G}= (\mathbf{V},\mathbf{E}, \mathbf{P})$, if the influence function meets Condition~\ref{condition:submodular}, then for $\forall \mathcal{S}_i,\mathcal{S}_j \subseteq \mathbf{V}$: 
    \begin{equation}
        f_{\mathcal{G}}(\mathcal{S}_i)-f_{\mathcal{G}}(\mathcal{S}_j)\leq\sum_{\mathbf{v}\in\mathcal{S}_i-\mathcal{S}_j}\psi_{\mathbf{v}}(\mathcal{S}_j)-\sum_{\mathbf{v}\in\mathcal{S}_j-\mathcal{S}_i}\psi_{\mathbf{v}}(\mathcal{S}_i\cup\mathcal{S}_j-\mathbf{v}),
    \end{equation}
where $\psi_{\mathbf{v}}(\mathcal{S}_j):=f_{\mathcal{G}}(\mathcal{S}_j\cup\mathbf{v})-f_{\mathcal{G}}(\mathcal{S}_j)$. 
\end{lem}
\begin{proof}
    The proof is inspired by~\cite{rolnick2014greedy}. We first let 
\begin{equation}
\label{eq:i2j}
\mathcal{S}_{i}-\mathcal{S}_{j} = \{\mathbf{a}_{1},...,\mathbf{a}_{r}\} 
\end{equation}
and 
\begin{equation}
\label{eq:j2i}
\mathcal{S}_{j}-\mathcal{S}_{i} = \{\mathbf{b}_{1},...,\mathbf{b}_{q}\},
\end{equation}
where $r\in\mathbbm{N}_+$ and $q\in\mathbbm{N}_+$. According to Eq.~(\ref{eq:i2j}), for the subsets $\mathcal{S}_i$ and $\mathcal{S}_j$, we have 
\begin{equation}
\label{eq:tem}
\mathcal{S}_{j} \cup \mathcal{S}_{i} = \mathcal{S}_{j} \cup \{\mathbf{a}_{1},...,\mathbf{a}_{r}\}.
\end{equation}
Afterward, we obtain
\begin{equation}\label{eq:lemma-step1}
    f_{\mathcal{G}}(\mathcal{S}_j\cup\mathcal{S}_i)-f_{\mathcal{G}}(\mathcal{S}_j)=f_{\mathcal{G}}(\mathcal{S}_{j} \cup \{\mathbf{a}_{1},...,\mathbf{a}_{r}\})-f_{\mathcal{G}}(\mathcal{S}_j).
\end{equation}
At a high level, Eq.~(\ref{eq:lemma-step1}) is to calculate the influence improvement of $\mathcal{S}_j$ after adding data points $\{\mathbf{a}_{1},...,\mathbf{a}_{r}\}$ into $\mathcal{S}_j$. As the influence improvement of adding one sequence of data points is equal to the sum of the influence improvement at each step, we have,
\begin{align}\label{eq:lemma-step3}
    & \quad f_{\mathcal{G}}(\mathcal{S}_j\cup\mathcal{S}_i)-f_{\mathcal{G}}(\mathcal{S}_j)\\\nonumber
    &=f_{\mathcal{G}}(\mathcal{S}_j\cup\mathbf{a}_1)-f_{\mathcal{G}}(\mathcal{S}_j)+\sum_{k=2}^r\left[f_{\mathcal{G}}(\mathcal{S}_j\cup\{\mathbf{a}_{1},...,\mathbf{a}_{k}\}) - f_{\mathcal{G}}(\mathcal{S}_{j} \cup \{\mathbf{a}_{1},...,\mathbf{a}_{k-1}\})\right]\\\nonumber
    &=\psi_{\mathbf{a}_1}(\mathcal{S}_j)+\sum_{k=2}^r\psi_{\mathbf{a}_k}(\mathcal{S}_{j} \cup \{\mathbf{a}_{1},...,\mathbf{a}_{k-1}\}).
\end{align}
Under Condition~\ref{condition:submodular}, as $\mathcal{S}_j\subset\mathcal{S}_j\cup\{\mathbf{a}_{1},...,\mathbf{a}_{k-1}\}$, we have 
\begin{align}
    f_{\mathcal{G}}(\mathcal{S}_j\cup\mathcal{S}_i)-f_{\mathcal{G}}(\mathcal{S}_j)&=\psi_{\mathbf{a}_1}(\mathcal{S}_j)+\sum_{k=2}^r\psi_{\mathbf{a}_k}(\mathcal{S}_{j} \cup \{\mathbf{a}_{1},...,\mathbf{a}_{k-1}\})\\\nonumber
    &\leq\sum_{k=1}^r\psi_{\mathbf{a}_k}(\mathcal{S}_j)=\sum_{\mathbf{a}\in\mathcal{S}_i-\mathcal{S}_j}\psi_{\mathbf{a}}(\mathcal{S}_j).
\end{align}
Similarly, 
\begin{align}\label{eq:lemma-step4}
    & \quad f_{\mathcal{G}}(\mathcal{S}_j\cup\mathcal{S}_i)-f_{\mathcal{G}}(\mathcal{S}_i)\\\nonumber
    &= \psi_{\mathbf{b}_1}(\mathcal{S}_i)+\sum_{k=2}^q\psi_{\mathbf{b}_k}(\mathcal{S}_{i} \cup \{\mathbf{b}_{1},...,\mathbf{b}_{k-1}\})\geq \sum_{k=1}^q \psi_{\mathbf{b}_k}(\mathcal{S}_i\cup\mathcal{S}_{j}-\mathbf{b}_k)=\sum_{\mathbf{b}\in\mathcal{S}_j-\mathcal{S}_i}\psi_{\mathbf{b}}(\mathcal{S}_i).
\end{align}
By subtracting (\ref{eq:lemma-step4}) from (\ref{eq:lemma-step3}), we have 
\begin{equation}
    f_{\mathcal{G}}(\mathcal{S}_i)-f_{\mathcal{G}}(\mathcal{S}_j)\leq\sum_{\mathbf{v}\in\mathcal{S}_i-\mathcal{S}_j}\psi_{\mathbf{v}}(\mathcal{S}_j)-\sum_{\mathbf{v}\in\mathcal{S}_j-\mathcal{S}_i}\psi_{\mathbf{v}}(\mathcal{S}_i\cup\mathcal{S}_j-\mathbf{v}).
\end{equation}
\end{proof}
\subsubsection{Lemma~2}
\begin{lem}\label{lem:2}
    Given a graph $\mathcal{G}= (\mathbf{V},\mathbf{E}, \mathbf{P})$, for any subset $\mathcal{S}\subset\mathbf{V}$ and any $\mathbf{v}\in\mathbf{V}$, the influence function $f_\mathcal{G}$ satisfies 
    \begin{equation}
        \psi_{\mathbf{v}}(\mathcal{S}) = f_{\mathcal{G}}(\mathcal{S} \cup \mathbf{v}) - f_{\mathcal{G}}(\mathcal{S})\geq  0
    \end{equation}
\end{lem}
\begin{proof}
    We consider two cases to finish the proof.

\textit{Case~1}~($\mathbf{v} \in \mathbf{V} \land \mathbf{v} \notin \mathcal{S}$). In this case, the influence improvement is at least 1 since $\mathbf{v}$ itself has been included, i.e., 
\begin{equation}
    \psi_{\mathbf{v}}(\mathcal{S})=f_{\mathcal{G}}(\mathcal{S}\cup\mathbf{v})-f_{\mathcal{G}}(\mathcal{S})\geq 1.
\end{equation}
\textit{Case~2}~($\mathbf{v} \in \mathbf{V} \land \mathbf{v} \in \mathcal{S}$). In this case, the influence improvement is 0 since $\mathbf{v}$ has already been included in $\mathcal{S}$, i.e., 
\begin{equation}
    \psi_{\mathbf{v}}(\mathcal{S})=f_{\mathcal{G}}(\mathcal{S}\cup\mathbf{v})-f_{\mathcal{G}}(\mathcal{S})=0.
\end{equation}
Combining the above two cases, we conclude that, for $\forall \mathbf{v}\in\mathbf{V}$, the influence function $f_{\mathcal{G}}$ satisfies 
\begin{equation}
    f_\mathcal{G}(\mathcal{S}\cup\mathbf{v})-f_\mathcal{G}(\mathcal{S})\geq 0.
\end{equation}

\end{proof}

\subsection{Proof of Proposition~\ref{prop1}}
\begin{proof}
Given a graph $\mathcal{G}=(\mathbf{V},\mathbf{E},\mathbf{P})$, for $\forall \mathcal{S}_i,\mathcal{S}_j\subset\mathbf{V}$, according to Lemma~\ref{lem:2}, we have 
\begin{equation}\label{eq:pro1-step2}
    \sum_{\mathbf{v}\in\mathcal{S}_i-\mathcal{S}_j}\psi_\mathbf{v}(\mathcal{S}_i\cup\mathcal{S}_j-\mathbf{v})\geq 0.
\end{equation}
Taking (\ref{eq:pro1-step2}) into Lemma~\ref{lem:1}, we have 
\begin{equation}\label{eq:pro1-step3}
    f_\mathcal{G}(\mathcal{S}_i)-f_\mathcal{G}(\mathcal{S}_j)\leq\sum_{\mathbf{v}\in\mathcal{S}_i-\mathcal{S}_j}\psi_\mathbf{v}(S_j).
\end{equation}
We use $\mathcal{S}_m^*$ to denote the optimal solution as discussed in the main paper. At any step $t$ in Algorithm~\ref{alg:greedy-influence-maximization}, we substitute $\mathcal{S}_m^*$ (resp. $\mathcal{S}_t$) into $\mathcal{S}_i$ (resp. $\mathcal{S}_j$) in (\ref{eq:pro1-step3}), we can derive 
\begin{equation}\label{eq:pro1-step4}
    f_\mathcal{G}(\mathcal{S}_m^*)\leq f_\mathcal{G}(\mathcal{S}_t)+\sum_{\mathbf{v}\in\mathcal{S}_m^*-\mathcal{S}_t}\psi_\mathbf{v}(\mathcal{S}_t).
\end{equation}
According to Condition~\ref{condition:submodular}, 
\begin{equation}\label{eq:pro1-step5}
    \psi_\mathbf{v}(\mathcal{S}_t)\geq\psi_\mathbf{v}(\mathcal{S}_{t+1})
\end{equation}
holds. Taking both (\ref{eq:pro1-step4}) and (\ref{eq:pro1-step5}) into (\ref{eq:pro1-step3}), we have for any $t$,
\begin{align}
\label{eq:pro1-step6}
f_{\mathcal{G}}(\mathcal{S}^{*}_{m})   &\leq  f_{\mathcal{G}}( \mathcal{S}_{t}) + m\psi_{t+1}.
\end{align}
\end{proof}

\subsection{Proof of Theorem~\ref{them1}} 
\begin{proof}
Recall that 
\begin{align}
\psi_{t} = f_{\mathcal{G}}(\mathcal{S}_{t}) -  f_{\mathcal{G}}(\mathcal{S}_{t-1}).
\end{align}
According to Proposition~\ref{prop1}, we have
\begin{align}
\label{eq:thm-step2}
f_{\mathcal{G}}(\mathcal{S}^{*}_{m})&-  f_{\mathcal{G}}(\mathcal{S}_{t}) \leq  m \psi_{t+1}= m(f_{\mathcal{G}}(\mathcal{S}_{t+1}) -  f_{\mathcal{G}}(\mathcal{S}_{t})).
\end{align}
Afterwards, (\ref{eq:thm-step2}) equals to,
\begin{align}
\label{eq:thm-step3}
&\quad \quad f_{\mathcal{G}}(\mathcal{S}^{*}_{m})-  f_{\mathcal{G}}(\mathcal{S}_{t})  -  (f_{\mathcal{G}}(\mathcal{S}^{*}_{m})-f_{\mathcal{G}}(\mathcal{S}_{t+1}) )
\geq   \frac{1}{m}(f_{\mathcal{G}}(\mathcal{S}^{*}_{m})-  f_{\mathcal{G}}(\mathcal{S}_{t})) \\ \nonumber
&\iff f_{\mathcal{G}}(\mathcal{S}^{*}_{m})-  f_{\mathcal{G}}(\mathcal{S}_{t+1}) \leq \frac{m-1}{m}(f_{\mathcal{G}}(\mathcal{S}^{*}_{m})-  f_{\mathcal{G}}(\mathcal{S}_{t}) ).
\end{align}

Based on (\ref{eq:thm-step3}), we have
\begin{align}
&\quad f_{\mathcal{G}}(\mathcal{S}^{*}_{m})-  f_{\mathcal{G}}(\mathcal{S}_{t+1}) \leq \frac{m-1}{m}(f_{\mathcal{G}}(\mathcal{S}^{*}_{m})-  f_{\mathcal{G}}(\mathcal{S}_{t})  ) \\ \nonumber
& \leq (\frac{m-1}{m})^2(f_{\mathcal{G}}(\mathcal{S}^{*}_{m})-  f_{\mathcal{G}}(\mathcal{S}_{t-1}) ) \\ \nonumber
& \leq ...  \leq (\frac{m-1}{m})^{t+1}(f_{\mathcal{G}}(\mathcal{S}_{*}^{m})-  f_{\mathcal{G}}(\mathcal{S}_{0})  ). 
\end{align}
Since $f_{\mathcal{G}}(\mathcal{S}_{0})  = f_{\mathcal{G}}(\emptyset)  = 0$, we have 
\begin{align}
\frac{f_{\mathcal{G}}(\mathcal{S}^{*}_{m})-  f_{\mathcal{G}}(\mathcal{S}_{t+1})}{f_{\mathcal{G}}(\mathcal{S}^{*}_{m})} \leq (\frac{m-1}{m})^{t+1}.
\end{align}
When Algorithm~\ref{alg:greedy-influence-maximization} terminates at step $t=m-1$, we have, 
\begin{align}
f_{\mathcal{G}}(\mathcal{S}_{m}) \geq (1-(1-1/m)^m) f_{\mathcal{G}}(\mathcal{S}^{*}_{m}).
\end{align}
\end{proof}

\newpage
\section{Supplementary Experimental Results}
\label{app:exp}
\subsection{Selected examples}\label{app:select_examples}

In Table~\ref{tab:generation_examples}, for illustration purposes, we provide a few examples from the selection by our method, when the annotation size is 18.

\begin{table*}[h]
\begin{center}
\resizebox{\textwidth}{!}{%
\begin{tabular}{p{0.21\linewidth} p{0.95\linewidth}}
\toprule
Dataset &  Input\\
\midrule
MRPC & \begin{tabular}[t]{@{}p{\linewidth}@{}}
a. Input: The two Democrats on the five-member FCC held a press conference to sway opinion against [...] \\
~~~~Output: not equivalent \\ 
a. Input: The report shows that drugs sold in Canadian pharmacies are manufactured in facilities approved by Health Canada [...] \\ 
~~~~Output: equivalent \\ 
c. Input: The chief merchandising officer decides what the store is going to sell [...] \\
~~~~Output: equivalent \\
\end{tabular} \\
\midrule

SST-5 & \begin{tabular}[t]{@{}p{\linewidth}@{}}
a. Input: plodding, poorly written, murky and weakly acted, the picture feels as if everyone making it lost their movie mojo. \\
~~~~Output: very negative \\
b. Input: duvall is strong as always .  \\
~~~~Output: very positive \\
c. Input: lohman adapts to the changes required of her , but the actress and director peter kosminsky never get the audience to break [...] \\
~~~~Output: neutral \\
\end{tabular} \\
\midrule
MNLI & \begin{tabular}[t]{@{}p{\linewidth}@{}}
a. Input: This prosperous city has many museums, including a well-endowed Musee des Beaux-Arts (Square Verdrel) [...] \\
~~~~Output: False \\
b. Input: Duhame, who today makes her living as a graphic designer and illustrator, calls her book [...] \\
~~~~Output: Inconclusive \\
c. Input: At the agency or program level, it included management's public commitment to reduce fraud and errors, as. Based on that information [...] \\
~~~~Output: True \\
\end{tabular} \\ 
\midrule
DBpedia & \begin{tabular}[t]{@{}p{\linewidth}@{}}
a. Input: Lars Nielsen (born 3 November 1960 in Copenhagen) is a Danish rower. \\
~~~~Output: athlete \\ 
b. Input: Calhoun County High School is a public secondary school in St. Matthews South Carolina USA. \\
~~~~Output: educational institution \\ 
c. Input: David Goldschmid (sometimes credited as Dave Goldschmid) is an American television writer and producer currently writing for the daytime drama General Hospital. \\
~~~~Output: artist \\ 
\end{tabular} \\
\midrule
RTE & \begin{tabular}[t]{@{}p{\linewidth}@{}}
a. Input: In sub-Saharan Africa about one in every 30 people is infected with HIV.. 30\% of the people infected with HIV live in Africa.. \\
~~~~Output: False \\ 
b. Input: The drawbacks of legalization do not imply that our current version of prohibition is the optimal drug strategy; it may well [...] \\
~~~~Output: False \\ 
c. Input: For example, the fields of Western farmers feed the United States and many other parts of the world, and India's irrigation [...] \\
~~~~Output: True \\ 
\end{tabular} \\
\midrule
HellaSwag & \begin{tabular}[t]{@{}p{\linewidth}@{}}
a. Input: The topic is Preparing salad. An illustrated egg, the website "startcooking com" and "vegetable salad" [...] \\ 
~~~~Output: is shown from above. \\  
b. Input: The topic is Pets and Animals. [header] How to treat an injured rabbit's paw [title] Identify sore hocks. [step] Pododermatitis [...] \\ 
~~~~Output: Once the condition has set in, though, you'll need to take quick action to treat the injury. Leaving [...] \\  
c. Input: The topic is Playing squash. Two men stand on a racquetball court. the men \\ 
~~~~Output: stretch then begin playing. \\  
\end{tabular} \\
\bottomrule
\end{tabular}
}
\end{center}
\caption{
For illustration purposes, under our method, we show randomly selected three examples from each of the six datasets in one same run (excluding the other three datasets due to their length) when the annotation budget is set to 18. 
}
\label{tab:generation_examples}
\end{table*}

\subsection{Visualization of selected examples}\label{appendix:umap}
Here we provide a umap~\cite{mcinnes2018umap-software} visualization of selected examples. To avoid the denseness, we choose the annotation budget as 5. The visualization can be checked in Figure~\ref{fig:generation_examples}. First, comparing subfigures (a) and (b), we can clearly see that the selection of Vote-$k$ is much biased, and our IDEAL can identify a subset that is more favorable to be a proxy of full data. Second, comparing subfigures (c) and (d), we can see that the selected subset by Vote-$k$ is distributed on the right of full data. By comparison, our IDEAL can select a subset that is distributed more uniformly.

\begin{figure}[h]
    \centering
    \begin{subfigure}[b]{0.45\textwidth}
        \includegraphics[width=\textwidth]{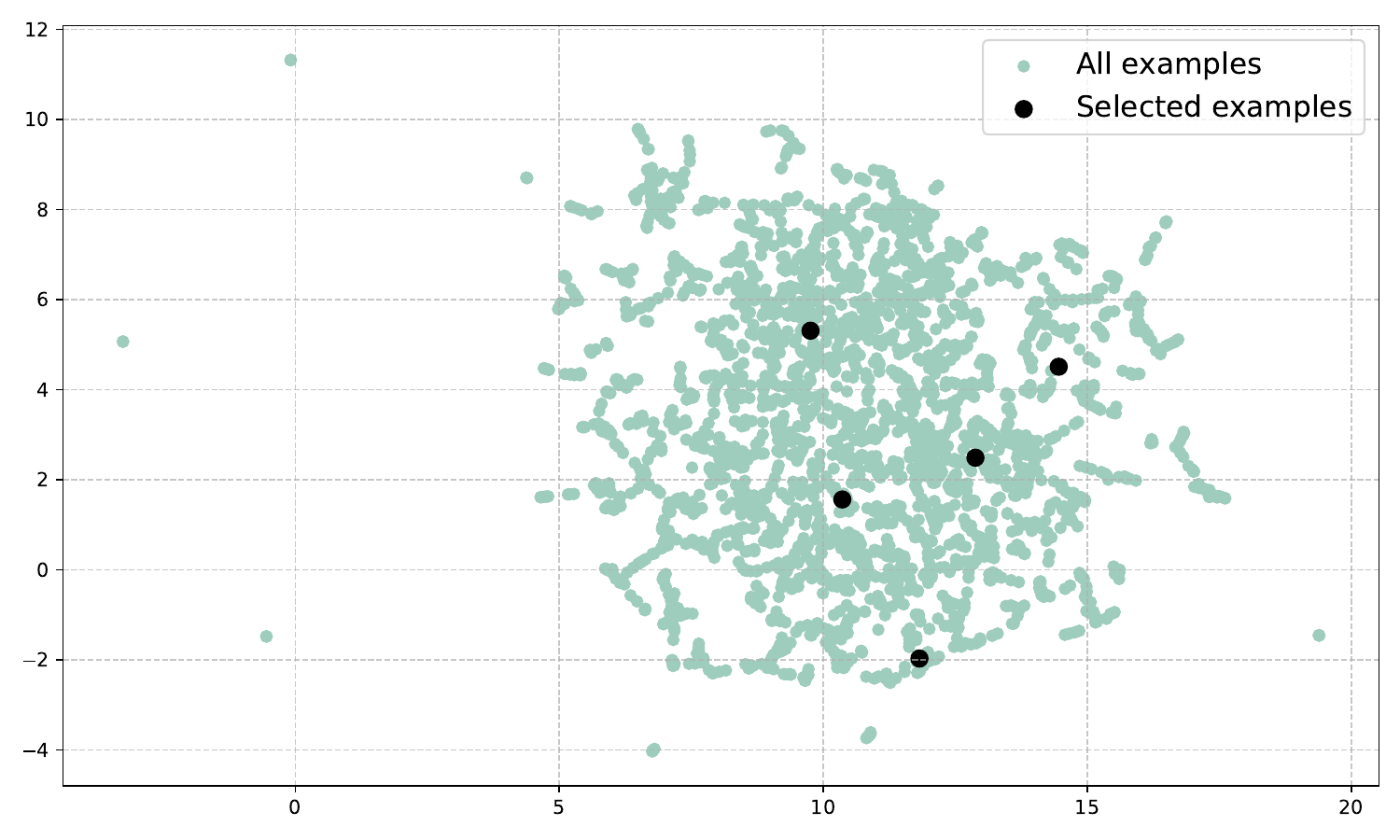}
        \caption{SST-5, Vote-$k$.}
    \end{subfigure}
    \begin{subfigure}[b]{0.45\textwidth}
    \includegraphics[width=\textwidth]{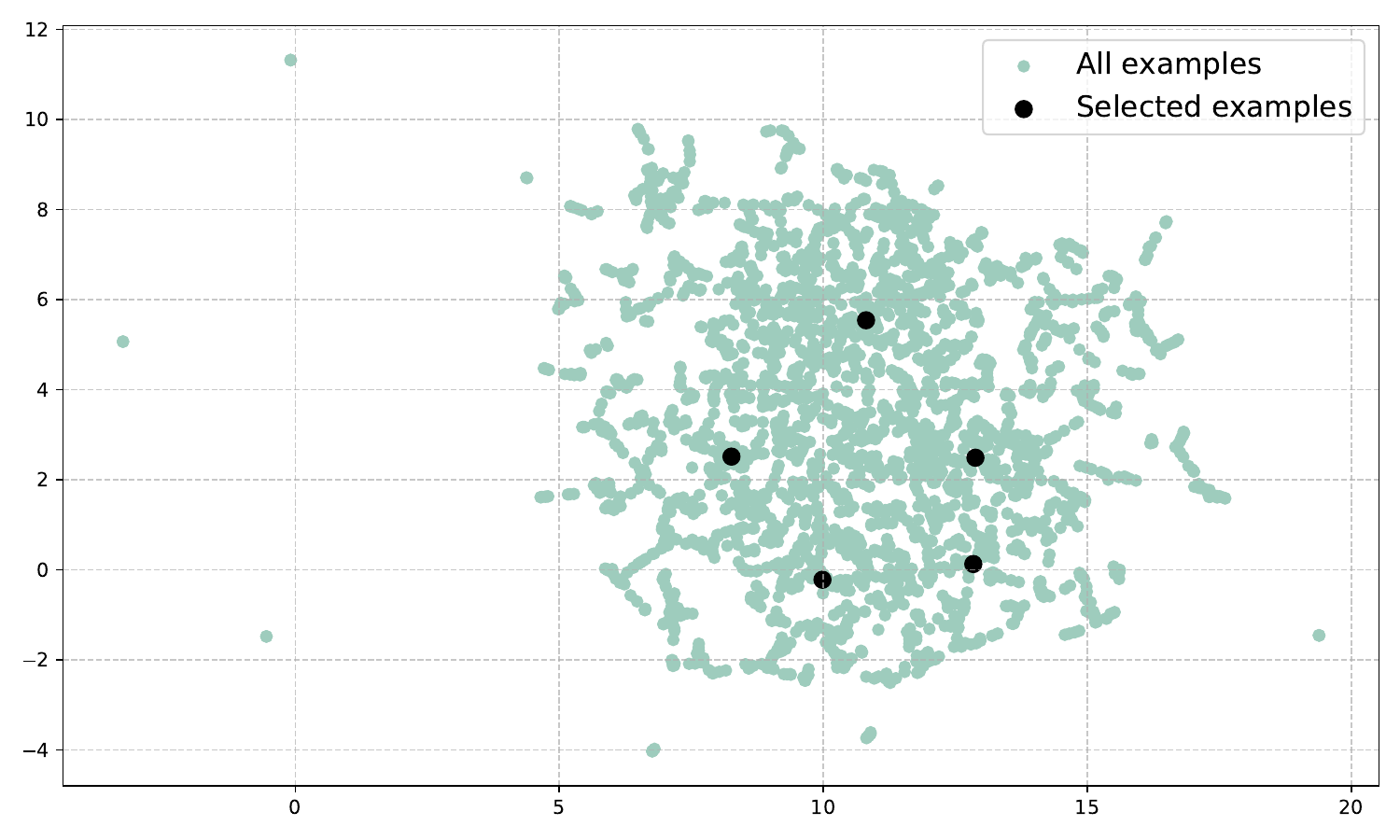}
    \caption{SST5, IDEAL.}
    \end{subfigure}
    \begin{subfigure}[b]{0.45\textwidth}
        \includegraphics[width=\textwidth]{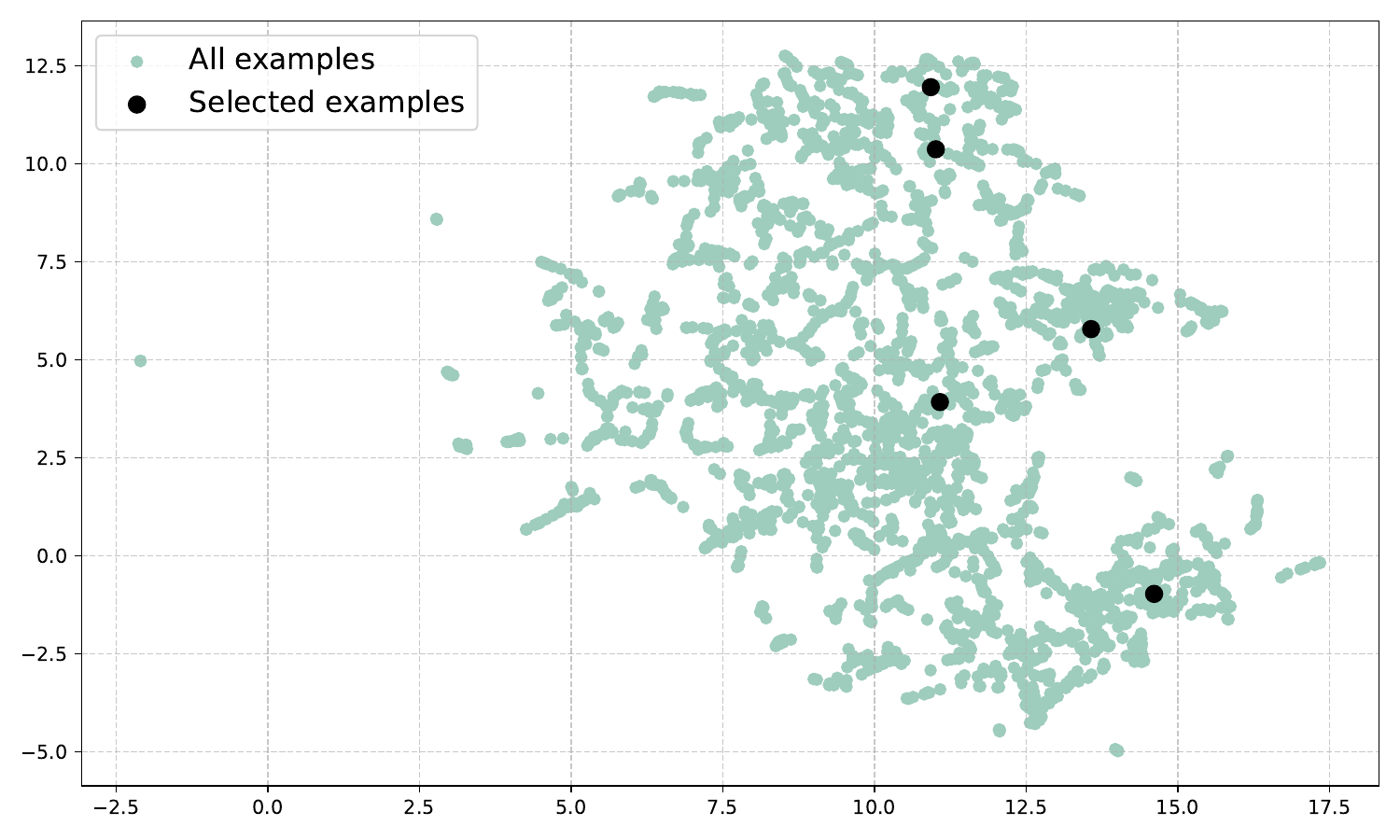}
    \caption{MNLI, Vote-$k$.}
    \end{subfigure}
    \begin{subfigure}[b]{0.45\textwidth}
    \includegraphics[width=\textwidth]{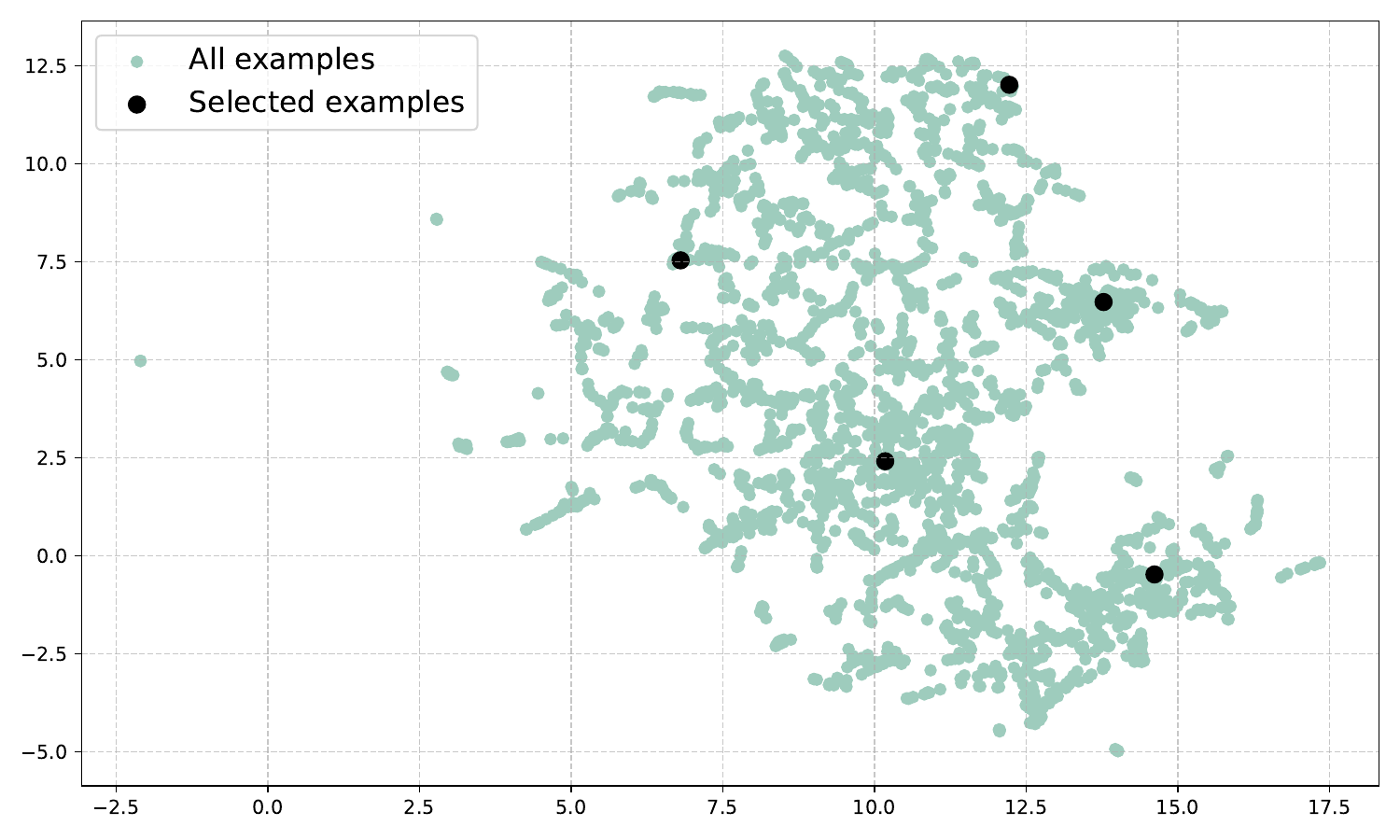}
    \caption{MNLI, IDEAL.}
    \end{subfigure}
    \caption{Umap~\cite{mcinnes2018umap-software} visualization to compare five selected examples from all examples using fully unsupervised methods Vote-$k$ and IDEAL~(ours). Compared with Vote-$k$, IDEAL could choose the examples to better represent the whole data rather than get involved in diversity and including outliers.}
    \label{fig:generation_examples}
\end{figure}

\subsection{Detailed experimental results in Table~\ref{tab:main_results}}\label{app:max_min}
\begin{table}[h!]
\addtolength{\tabcolsep}{-2pt} 
\renewcommand{\arraystretch}{1} 
\footnotesize 
\centering
\begin{tabular}{ cc   c  cccccc }
\toprule

\multicolumn{2}{c}{\textbf{Method}}
& \textbf{MRPC}
& \textbf{SST-5}
& \textbf{MNLI}
& \textbf{DBpedia}
& \textbf{RTE}
\\
\midrule[.1em]

100
&Random
& 64.3/68.4/58.6
&  49.6/51.1/47.2
&  38.2/40.2/36.7
&  89.8/91.0/88.2
&  55.3/55.9/55.1
\\

100
&\Votek
& 64.6/68.8/62.1
& 46.6/47.2/46.1
& 38.9/43.8/35.5
& 89.2/89.8/88.7
& 57.6/58.2/57.4
\\

\rowcolor{gray!30} 
100
&IDEAL
& \textbf{66.4}/67.9/64.8
& \textbf{51.4}/53.5/49.6
& \textbf{41.0}/41.4/40.2
& \textbf{90.6}/91.4/89.5
& \textbf{58.9}/60.9/57.4
\\

\midrule[.05em]

18
&Random
& 57.4/68.8/39.8
& 42.9/46.9/39.1
& 37.8/39.4/35.2
&  85.2/87.5/83.9
& 57.9/58.9/57.0
\\

18
&\Votek
& 61.1/67.2/52.7
& 41.7/45.7/37.1
& 39.1/43.8/32.0
& 89.9/94.1/87.1
& 58.2/58.9/57.8
\\
\rowcolor{gray!30} 
18
&IDEAL
& \textbf{63.0}/63.7/62.5
& \textbf{43.2}/45.7/39.5
& \textbf{40.0}/41.8/37.1
& \textbf{90.1}/90.2/89.8
& \textbf{59.4}/60.9/57.8
\\

\bottomrule
\end{tabular}
\caption{Mean/maximum/minimum evaluation results of all methods on classification tasks in Table~\ref{tab:main_results} over three different trials. The best mean result in each case is \textbf{bolded}.}
\label{tab:app-main-results-1}
\end{table}

\begin{table}[h!]
\addtolength{\tabcolsep}{-0.1pt} 
\renewcommand{\arraystretch}{1} 
\footnotesize 
\centering
\begin{tabular}{cc   c  ccccc }
\toprule

\multicolumn{2}{c}{\textbf{Method}}
& \textbf{HellaSwag}
& \textbf{MWoZ}
& \textbf{GeoQ}
& \textbf{Xsum}
\\
\midrule[.1em]

100
&Random
& 66.7/70.3/64.1
& 39.9/48.4/39.9
& 55.3/57.8/53.1
& 15.3/16.4/14.8
\\

100
&\Votek
& 67.9/69.9/64.0
& 48.3/50.8/46.9
& \textbf{58.8}/60.5/57.0
& 17.2/17.6/16.4
\\
\rowcolor{gray!30} 
100
&IDEAL
& \textbf{68.6}/71.9/65.2
& \textbf{52.2}/55.9/49.1
& 58.2/60.5/54.7
& \textbf{19.9}/20.2/19.5
\\
\midrule[.05em]

18
&Random
& 66.0/68.8/63.7
& 37.0/46.5/28.1
& 47.5/49.2/44.9
& 13.6/14.5/12.5
\\

18
&\Votek
& 66.5/71.9/62.5
& 37.7/43.8/32.4
& 50.9/54.3/47.7
& 15.2/16.0/14.5
\\
\rowcolor{gray!30} 
18
&IDEAL
& \textbf{67.1}/71.9/64.5
& \textbf{38.5}/47.3/30.9
& \textbf{52.0}/53.9/50.8
& \textbf{19.6}/20.2/18.9
\\

\bottomrule
\end{tabular}
\caption{Mean/maximum/minimum evaluation results of all methods on multi-choice, dialogue, and generation tasks in Table~\ref{tab:main_results} over three different trials. The best mean result in each case is \textbf{bolded}.}
\label{tab:app-main-results-2}
\end{table}

In the main paper~(Table~\ref{tab:main_results}), we report the mean evaluation results for different methods over three random trials. Here we provided the detailed results of Table~\ref{tab:main_results} with mean/maximum/minimum values. 
We can observe IDEAL achieves stable results compared with baselines. Moreover, the worst-case performance of IDEAL is obviously better compared with baselines in most cases. 

\subsection{Label distributions in selective annotations}\label{appendix:label_statistic}

Recall that the process of selective annotations is based entirely on similarities derived from sentence embeddings without labels. Therefore, we investigate whether the selected examples have label skew. Under an annotation budget of 100, we collect all selected examples in three classification tasks (MRPC, MNLI, and RTE) and show the numbers of different labels for different methods in Table~\ref{tab:label_distribution}. We also present the label statistics of the original training data.
We observe that random selection shows a great variance. However, in an ideal case, it should achieve a similar distribution as the original training data. Notably, IDEAL achieves the smallest ratio between the numbers of the most frequent class and the least frequent class in 2 out of 3 cases (MNLI and RTE). This demonstrates that IDEAL can indeed balance the label distribution in the selected subset and mitigate the problem of label skew.

\begin{table}[!t]
\centering
\begin{tabular}{cccccccc}
\toprule
\multirow{2}{*}{\textbf{Method}} & \multicolumn{2}{c}{\textbf{MRPC}} & \multicolumn{3}{c}{\textbf{MNLI}} & \multicolumn{2}{c}{\textbf{RTE}} \\ 
\cmidrule(lr){2-3}
\cmidrule(lr){4-6}
\cmidrule(lr){7-8}
                        & Equivalent     & Not equivalent    & True      & Inconclusive & False   & True     & False   \\ \hline
Original                  &2023            &977               &1051       & 965          & 984     & 1241     & 1249    \\
Random                  & 70             & 30                & 30        & 39           & 31      & 56       & 44      \\
Vote-$k$                & 64             & 36                & 27        & 35           & 38      & 46       & 54      \\
\rowcolor{gray!30}  IDEAL   &  65           & 35                & 37        & 34           & 29      & 49       & 51      \\ \midrule[.1em]
\end{tabular}
\caption{The numbers of different labels in the selected examples for different methods. ``Original'' denotes the label statistics of the original dataset. Under the annotation budget 100, IDEAL achieves the smallest ratio between the numbers of the most frequent class and the least frequent class in 2 out of 3 cases (MNLI and RTE), implying IDEAL can indeed mitigate the label skew problem.}
\label{tab:label_distribution}
\end{table}

\begin{table}[!t]
\centering
\begin{tabular}{ccccccc}
\toprule
\multirow{2}{*}{\textbf{Method}} & \multicolumn{2}{c}{\textbf{MRPC}} & \multicolumn{2}{c}{\textbf{SST-5}} & \multicolumn{2}{c}{\textbf{RTE}} \\ 
\cmidrule(lr){2-3}
\cmidrule(lr){4-5}
\cmidrule(lr){6-7}
                        & Mean     & Std    & Mean      & Std   & Mean     & Std   \\ \hline
Random                  &  44.2           &  0.02               & 45.4        &  0.02     & 57.3       & 0.02      \\
Vote-$k$                & 52.7             &  0.03               & 38.0         & 0.01      & 58.6        & 0.02       \\
\rowcolor{gray!30}  IDEAL   & 65.5            &  0.01               & 46.6     & 0.01      & 57.5       & 0.02      \\ \midrule[.1em]
\end{tabular}
\caption{The average performance of different methods by permuting the order of prompts for each test instance 10 times. We conduct experiments on MRPC, SST-5, and RTE datasets and report the average results with standard deviation. We can observe the subset selected by IDEAL achieves the best performance compared with baselines in 2 out of 3 cases. IDEAL also achieves the lowest standard deviations in all evaluations, which suggests IDEAL is a more stable and robust method against the order of prompts.}
\label{tab:prompts_order}
\end{table}

\subsection{Prompt order in selective annotation}\label{prompts_order}
As pointed out by ~\cite{lu2021fantastically}, the performance of in-context learning is influenced not only by the selection of prompts but also by the order in which the prompts are presented to models. 
Although this work focuses solely on selective annotation problems, we are interested in exploring whether the selected subset can still lead to better performance when the order of prompts is permuted.
Under an annotation budget of 18, we first retrieve prompts for each test instance from selected subsets achieved by different selective annotation methods. We then permute the order of prompts for each test instance 10 times, resulting in 10 different experimental trials. We show the average performance of these 10 trials and make a comparison between different selective annotation methods.
We conduct experiments on MRPC, SST-5, and RTE datasets and present the results in Table~\ref{tab:prompts_order}.
The results show that IDEAL outperforms baselines in 2 out of 3 cases, suggesting that our method can choose more stable and robust subsets against changed prompt orders.


\section{Supplementary Descriptions of Experimental Settings}
\subsection{Details of datasets}\label{app:datasets}
In this paper, to demonstrate the superiority of our method, we employ 9 datasets which can be categorized into 4 different tasks, including \textit{classification} (MRPC~\cite{dolan2004unsupervised}, SST-5~\cite{socher2013recursive}, MNLI~\cite{williams2017broad}, DBpedia~\cite{lehmann2015dbpedia}, and RTE~\cite{bentivogli2009fifth}), \textit{multi-choice} (HellaSwag~\cite{zellers2019hellaswag}), \textit{dialogue} (MWoZ~\cite{budzianowski2018multiwoz}), and \textit{generation} (GeoQuery~\cite{zelle1996learning} and Xsum~\cite{narayan2018don}). We list the datasets and the models used in Table~\ref{tab:datasets_model}.

\begin{table*}[!h]
\begin{adjustbox}{width=0.96\linewidth}
\begin{tabular}{@{}l@{}l@{}c@{}c@{}}
\toprule
& Datasets & Task & Models \\ 
\midrule[.005em]
\multirow{5}*{\textbf{Classification}} & MRPC~\cite{dolan2004unsupervised} & Paraphrase Detection
&
GPT-Neo, GPT-J, GPT-3.5-Turbo
\\
& SST-5~\cite{socher2013recursive} & Sentiment Analysis & GPT-J
\\
& DBpedia~\cite{lehmann2015dbpedia} & Topic Classification & 
GPT-J 
\\
& RTE~\cite{bentivogli2009fifth}) & \ \ Natural Language Inference \ \ & 
GPT-Neo, GPT-J, GPT-3.5-Turbo
\\
& MNLI~\cite{williams2017broad} & Natural Language Inference & 
GPT-Neo, GPT-J, GPT-3.5-Turbo
\\

\midrule[.005em]
\multirow{1}*{\textbf{Multiple-Choice}} \ \  & HellaSwag~\cite{zellers2019hellaswag} & Commonsense \par Reasoning &
GPT-J
\\
\midrule[.005em]
\multirow{1}*{\textbf{Dialogue}} & MWoZ~\cite{budzianowski2018multiwoz} & Dialogue State Tracking & 
Text-davinci-002
\\
\midrule[.005em]
\multirow{2}*{\textbf{Generation}} & GeoQuery~\cite{zelle1996learning} & Semantic Parsing &
Text-davinci-002
\\
& Xsum~\cite{narayan2018don} & Summarization & 
GPT-J
\\
\bottomrule
\end{tabular}
\end{adjustbox}
\caption{
The datasets and corresponding models used in our experiments. We use GPT-J 6B and Text-davinci-002 by default. Other large language models are explored in \S\ref{other_evaluation}. 
}

\label{tab:datasets_model}
\end{table*}

To help readers better understand the datasets and tasks, for each of these datasets, we also list one example including both the input and output.

\subsubsection{MRPC}

\textbf{Input}

\begin{lstlisting}[language={}, gobble=0]
    Are the following two sentences 'equivalent' or 'not equivalent'?\nA federal judge yesterday disconnected a new national \" do-not-call \" list , just days before it was to take effect , saying the agency that created it lacked the authority ..\nA federal judge yesterday struck down the national do-not-call registry slated to take effect Oct. 1 , ruling the Federal Trade Commission had no authority to create the list ..\nanswer:
\end{lstlisting}
\textbf{Output}
\begin{lstlisting}[language={}]
    equivalent
\end{lstlisting}
\subsubsection{SST-5}
\textbf{Input}

\begin{lstlisting}[language={}]
    How do you feel about the following sentence?\nsmug , artificial , ill-constructed and fatally overlong ... it never finds a consistent tone and lacks bite , degenerating into a pious , preachy soap opera .\nanswer:
\end{lstlisting}
\textbf{Output}
\begin{lstlisting}[language={}]
    neutral
\end{lstlisting}
\subsubsection{MNLI}
\textbf{Input}
\begin{lstlisting}[language={}]
    yeah well the Cardinals i don't know  i think the Cowboys probably have a a better team they just at the end of the season the kind of got messed up with Aikman getting hurt because uh Laufenberg just couldn't never really get it together at all of course he sat along the sidelines all season he never really got in a game never did a whole lot. Based on that information, is the claim The Cowboys should have started Laufenberg all season.  \"True\", \"False\", or \"Inconclusive\"?\nanswer:
\end{lstlisting}
\textbf{Output}
\begin{lstlisting}[language={}]
    Inconclusive
\end{lstlisting}
\subsubsection{Dbpedia}
\textbf{Input}
\begin{lstlisting}[language={}, gobble=0,]
    title: V\u00edctor David Loubriel; content:  V\u00edctor David Loubriel Ort\u00edz is a Puerto Rican politician and former member of the Senate of Puerto Rico for the New Progressive Party (PNP).Loubriel presented his candidacy for the Senate of Puerto Rico before 2004. He ran for a candidate slot in the 2003 primaries obtaining the most votes in his district (Arecibo).In the 2004 general election Loubriel won a seat in the 23rd Senate of Puerto Rico to represent the district of Arecibo along with Jos\u00e9 Emilio Gonz\u00e1lez Vel\u00e1zquez.
\end{lstlisting}
\textbf{Output}
\begin{lstlisting}
    office holder
\end{lstlisting}
\subsubsection{RTE}
\textbf{Input}
\begin{lstlisting}[language={}]
    MEXICO CITY (Reuters) - A deadly strain of swine flu never seen before has broken out in Mexico, killing as many as 60 people and raising fears it is spreading across North America. The World Health Organization said it was concerned about what it called 800 \"influenza-like\" cases in Mexico, and also about a confirmed outbreak of a new strain of swine flu in the United States. It said about 60 people had died in Mexico. Mexico's government said it had confirmed that at least 16 people had died of the swine flu in central Mexico and that there could be another 45 fatal victims..\nquestion: 800 Mexicans have been affected by a new form of swine influenza.. True or False?\nanswer:
\end{lstlisting}

\textbf{Output}
\begin{lstlisting}
    True
\end{lstlisting}
\subsubsection{HellaSwag}
\textbf{Input}
\begin{lstlisting}[language={}]
    The topic is Work World. [header] How to become a high school social studies teacher [title] Obtain your bachelor's degree in education. [step] All schools will require you to obtain at least your bachelor's degree in education. This degree will be proof that you are capable of delivering information to students using the current educational best practices.
\end{lstlisting}
\textbf{Output}
\begin{lstlisting}[language={}]
    Make sure you've fully completed all of your course work and obtained your bachelor's degree before you seek certification or employment. [substeps] Your electives should be based in social studies courses.
\end{lstlisting}

\subsubsection{MultiWoz}
\textbf{Input}
\begin{lstlisting}[language={}]
    CREATE TABLE hotel(
      name text,
      ......,
      internet text CHECK (internet IN (dontcare, yes, no))
    )
    /*
    4 example rows:
    SELECT * FROM hotel LIMIT 4;
    name  pricerange  type  parking book_number_of_days book_day  book_people 
    area  stars internet
    a and b guest house moderate guest house dontcare 3 friday 5 east 4 yes
    ......
    /*
    ......
    -- Using valid SQLite, answer the following multi-turn conversational 
    questions for the tables provided above.
    Example #1
    [context] hotel-area: west, hotel-stars: 3, hotel-internet: yes
    [system] the hobsons house is available in that area .
    Q: [user] that sounds like it will work . can i book that for 3 nights
    starting wednesday ?
    SQL: SELECT * FROM hotel WHERE book_day = wednesday AND book_people = 1 
    AND book_number_of_days = 3 AND name = hobsons house;
    ......
\end{lstlisting}
\textbf{Output}
\begin{lstlisting}
    hotel WHERE book_day = wednesday AND book_number_of_days = 4 AND name = 
    warkworth house;
\end{lstlisting}

\subsubsection{GeoQ}
\textbf{Input}
\begin{lstlisting}[language={}]
    CREATE TABLE "border_info" ("state_name" text, "border" text)
    /*
    state_name    border
       alabama tennessee
       alabama   georgia
       alabama   florida
    */
    ......
    -- Using valid SQLite, answer the following questions for the tables 
    provided above.
    ......
    -- what is the longest river in the state with the highest point
    SELECT
\end{lstlisting}
\textbf{Output}
\begin{lstlisting}
    RIVERalias0.RIVER_NAME FROM HIGHLOW AS HIGHLOWalias0, RIVER AS 
    RIVERalias0 WHERE HIGHLOWalias0.HIGHEST_ELEVATION = (SELECT MAX(
    HIGHLOWalias1.HIGHEST_ELEVATION) FROM HIGHLOW AS HIGHLOWalias1 ) AND
    RIVERalias0.TRAVERSE = HIGHLOWalias0.STATE_NAME ORDER BY RIVERalias0.
    LENGTH DESC LIMIT 1;
\end{lstlisting}
\subsubsection{Xsum}
\textbf{Input}
\begin{lstlisting}[language={}]
    For decades, large numbers of Haitians have migrated - many of them without papers - to the Dominican Republic, to escape the poverty and lack of employment in their homeland.\nIn 2013, the Dominican Republic's highest court ruled that children born there to undocumented migrants were not automatically eligible for Dominican nationality.
    ......
    \nThere he strips the trees for firewood to make charcoal, to sell to Dominican traders for a few dollars.\nHe knows the practice damages the fertility of the soil, but it's the only available source of income.\n\"This is the only way we can survive,\" he says, motioning at his family, stuck inside the world's forgotten migrant crisis.\nYou can hear more of Will Grant's report on Heart and Soul on the BBC World Service.
\end{lstlisting}
\textbf{Output}
\begin{lstlisting}
    Immigration has long been a divisive issue on Hispaniola, the Caribbean island shared by Haiti and the Dominican Republic.
\end{lstlisting}

\subsection{Implementation details}\label{app:imp_details}

\textbf{General experimental conditions.} 
We primarily use PyTorch~\cite{paszke2019pytorch} to implement our algorithm and baselines. 
 For GPT-3.5-Turbo, we perform the experiments by calling the OpenAI API using a single Intel Xeon CPU.
The GPT-J 6B and GPT-Neo 2.7B models are from the Huggingface transformer library~\cite{wolf2019huggingface}. We run all our experiments of GPT-J 6B and GPT-Neo 2.7B on a single NVIDIA Tesla V100 (32GB) GPU.

\textbf{Details of getting unlabeled data.}
Since obtaining unlabeled examples in realistic scenarios is also a high-variance process, we follow the same setting as~\cite{su2022selective} to simulate the realistic setting. We perform selective annotations from 3k instances that are randomly sub-sampled from training data for each task. For each experiment, we repeat the sub-sampling process three times and average the results over all trials to ensure comprehensive evaluations. 

\textbf{Details of the graph construction.}
Except for the illustration experiment in Figure~\ref{fig:demon}, we construct the directed graph for all unlabeled data by connecting each vertex to its 10 nearest successors ($k=10$). It is important to note that a larger $k$ will lead to an increase in the computation cost. We have chosen this setting because it provides good performance while maintaining efficient computation costs.
For Figure~\ref{fig:demon}, we construct the graph by connecting each vertex to its 3 nearest successors in order to avoid denseness.

\textbf{Details of Algorithm~\ref{alg:influence_function}.}
Considering the randomness of the diffusion process, when quantifying the influence of the subset, we run Algorithm~\ref{alg:influence_function} 10 times and use the averaged influence value. Note that we also calculate the time cost in this repeated process when reporting the final results in the main paper.  As shown in Figure~\ref{fig:cost}, our algorithm is still more effective than Vote-$k$.

\section{Limitations}
\textbf{Memory cost.}  Although in-context learning tasks avoid the heavy parameter update process, they still require a large amount of memory to load models. For example, loading GPT-J 6B into a GPU requires about 23GB GPU memory, without considering the size of the dataset. This is a relatively high cost for individual researchers.

\textbf{Time cost of Auto-IDEAL.} Although Auto-IDEAL achieves even better performance than IDEAL, it has the same drawback as Vote-$k$. That is to say, when making automatic annotations, it incurs the cost of making predictions for all unlabeled data. Future work may study how to maintain superior performance while reducing the automatic annotation cost of IDEAL at the same time. Compared with~\cite{su2022selective}, we do not evaluate NQ~\cite{kwiatkowski2019natural} due to budget constraints.

\end{document}